\documentclass{article}
\usepackage{amsmath,amssymb,amsfonts,amsthm}
\usepackage{algorithm,algorithmic}
\usepackage{authblk}
\usepackage{graphicx}
\usepackage{pgfplots}
\usepackage{natbib}
\usepackage{url}
\usepackage{subcaption}
\usepackage{tikz}
\usepackage{xifthen}
\usepackage{tkz-euclide}

\newtheorem{lemma}{Lemma}
\newtheorem{theorem}{Theorem}

\renewcommand{\d}[1]{\operatorname{d}\!{#1}}
\newcommand{\N}[1]{\mathcal{N}\!\left({#1}\right)}

\pgfplotsset{compat=1.16}
\usetikzlibrary{intersections}
\pgfdeclarelayer{bg}    
\pgfsetlayers{bg,main}  
\definecolor{beige}{RGB}{245, 245, 220}
\definecolor{darkgrey}{RGB}{65, 65, 65}
\definecolor{lightgrey}{RGB}{250, 250, 250}

\newcommand*\smallcircled[1]{\tikz[baseline=(char.base)]{
            \node[shape=circle,draw,inner sep=0pt, text width=2mm, align=center, minimum width=2.5mm] (char) {\scriptsize{#1}};}}
\newcommand*\smalldarkcircled[1]{\tikz[baseline=(char.base)]{
            \node[shape=circle,draw,inner sep=0pt, text width=2mm, align=center, fill=darkgrey, text=white, font=\bfseries, minimum width=2.5mm] (char) {\scriptsize{#1}};}}

\usetikzlibrary{calc, arrows, fit, positioning, patterns, decorations.pathreplacing, shapes}
\tikzstyle{dash} = [dashed, -latex,>=latex]
\tikzstyle{line} = [draw, -latex,>=latex]
\tikzstyle{box} = [draw, minimum size=.8cm]
\tikzstyle{smallbox} = [draw, minimum size=5mm]
\tikzstyle{mediumbox} = [draw, minimum size=6mm]
\tikzstyle{roundbox} = [draw, circle, inner sep=0pt, minimum size=.8cm]
\tikzstyle{clamped} = [draw, fill=black, minimum size=0.15cm]
\tikzstyle{msgcircle} = [shape=circle, draw, inner sep=0pt, minimum size=5mm, fill=white]
\tikzstyle{darkmsgcircle} = [shape=circle, draw, inner sep=0pt, minimum size=5mm, fill=darkgrey, text=white, font=\bfseries]

\newcommand{\msg}[6]{
      \ifthenelse{\isin{#1}{left} \AND \isin{#2}{down}}{
            \coordinate (anchor) at ($({#3})!{#5}!({#4})$);
            \node[msgcircle, xshift=-5.5mm] at (anchor) {#6};
            \node[xshift=-1.5mm] at (anchor) {$\downarrow$};
      }{}
      \ifthenelse{\isin{#1}{right} \AND \isin{#2}{down}}{
            \coordinate (anchor) at ($({#3})!{#5}!({#4})$);
            \node[msgcircle, xshift=5.5mm] at (anchor) {#6};
            \node[xshift=1.5mm] at (anchor) {$\downarrow$};
      }{}

      \ifthenelse{\isin{#1}{down} \AND \isin{#2}{right}}{
            \coordinate (anchor) at ($({#3})!{#5}!({#4})$);
            \node[msgcircle, yshift=-6.0mm] at (anchor) {#6};
            \node[yshift=-2.0mm] at (anchor) {$\rightarrow$};
      }{}
      \ifthenelse{\isin{#1}{up} \AND \isin{#2}{right}}{
            \coordinate (anchor) at ($({#3})!{#5}!({#4})$);
            \node[msgcircle, yshift=6.0mm] at (anchor) {#6};
            \node[yshift=2.0mm] at (anchor) {$\rightarrow$};
      }{}

      \ifthenelse{\isin{#1}{down} \AND \isin{#2}{left}}{
            \coordinate (anchor) at ($({#3})!{#5}!({#4})$);
            \node[msgcircle, yshift=-6.0mm] at (anchor) {#6};
            \node[yshift=-2.0mm] at (anchor) {$\leftarrow$};
      }{}
      \ifthenelse{\isin{#1}{up} \AND \isin{#2}{left}}{
            \coordinate (anchor) at ($({#3})!{#5}!({#4})$);
            \node[msgcircle, yshift=6.0mm] at (anchor) {#6};
            \node[yshift=2.0mm] at (anchor) {$\leftarrow$};
      }{}

      \ifthenelse{\isin{#1}{left} \AND \isin{#2}{up}}{
            \coordinate (anchor) at ($({#3})!{#5}!({#4})$);
            \node[msgcircle, xshift=-5.5mm] at (anchor) {#6};
            \node[xshift=-1.5mm] at (anchor) {$\uparrow$};
      }{}
      \ifthenelse{\isin{#1}{right} \AND \isin{#2}{up}}{
            \coordinate (anchor) at ($({#3})!{#5}!({#4})$);
            \node[msgcircle, xshift=5.5mm] at (anchor) {#6};
            \node[xshift=1.5mm] at (anchor) {$\uparrow$};
      }{}
}

\newcommand{\darkmsg}[6]{
      \ifthenelse{\isin{#1}{left} \AND \isin{#2}{down}}{
            \coordinate (anchor) at ($({#3})!{#5}!({#4})$);
            \node[darkmsgcircle, xshift=-5.5mm] at (anchor) {#6};
            \node[xshift=-1.5mm] at (anchor) {$\downarrow$};
      }{}
      \ifthenelse{\isin{#1}{right} \AND \isin{#2}{down}}{
            \coordinate (anchor) at ($({#3})!{#5}!({#4})$);
            \node[darkmsgcircle, xshift=5.5mm] at (anchor) {#6};
            \node[xshift=1.5mm] at (anchor) {$\downarrow$};
      }{}

      \ifthenelse{\isin{#1}{down} \AND \isin{#2}{right}}{
            \coordinate (anchor) at ($({#3})!{#5}!({#4})$);
            \node[darkmsgcircle, yshift=-6.0mm] at (anchor) {#6};
            \node[yshift=-2.0mm] at (anchor) {$\rightarrow$};
      }{}
      \ifthenelse{\isin{#1}{up} \AND \isin{#2}{right}}{
            \coordinate (anchor) at ($({#3})!{#5}!({#4})$);
            \node[darkmsgcircle, yshift=6.0mm] at (anchor) {#6};
            \node[yshift=2.0mm] at (anchor) {$\rightarrow$};
      }{}

      \ifthenelse{\isin{#1}{down} \AND \isin{#2}{left}}{
            \coordinate (anchor) at ($({#3})!{#5}!({#4})$);
            \node[darkmsgcircle, yshift=-6.0mm] at (anchor) {#6};
            \node[yshift=-2.0mm] at (anchor) {$\leftarrow$};
      }{}
      \ifthenelse{\isin{#1}{up} \AND \isin{#2}{left}}{
            \coordinate (anchor) at ($({#3})!{#5}!({#4})$);
            \node[darkmsgcircle, yshift=6.0mm] at (anchor) {#6};
            \node[yshift=2.0mm] at (anchor) {$\leftarrow$};
      }{}

      \ifthenelse{\isin{#1}{left} \AND \isin{#2}{up}}{
            \coordinate (anchor) at ($({#3})!{#5}!({#4})$);
            \node[darkmsgcircle, xshift=-5.5mm] at (anchor) {#6};
            \node[xshift=-1.5mm] at (anchor) {$\uparrow$};
      }{}
      \ifthenelse{\isin{#1}{right} \AND \isin{#2}{up}}{
            \coordinate (anchor) at ($({#3})!{#5}!({#4})$);
            \node[darkmsgcircle, xshift=5.5mm] at (anchor) {#6};
            \node[xshift=1.5mm] at (anchor) {$\uparrow$};
      }{}
}

\title{Chance-Constrained Active Inference}

\author[$*$]{Thijs van de Laar}
\author[$*$]{\.{I}smail \c{S}en\"{o}z}
\author[$\dagger$]{Ay\c ca \"Oz\c celikkale}
\author[$\ddagger$]{Henk Wymeersch}

\affil[$*$]{Dept. of Electrical Engineering, Eindhoven University of Technology, The Netherlands}
\affil[$\dagger$]{Dept. of Electrical Engineering, Uppsala University, Sweden}
\affil[$\ddagger$]{Dept. of Electrical Engineering, Chalmers University of Technology, Sweden}

\begin{document}

\maketitle

\begin{abstract}
Active Inference (ActInf) is an emerging theory that explains perception and action in biological agents, in terms of minimizing a free energy bound on Bayesian surprise. Goal-directed behavior is elicited by introducing
prior beliefs on the underlying generative model. In contrast to prior beliefs, which constrain all realizations of a random variable, we propose an alternative approach through chance constraints, which allow for a (typically small) probability of constraint violation, and demonstrate how such constraints can be used as intrinsic drivers for goal-directed behavior in ActInf.
We illustrate how chance-constrained ActInf weights all imposed (prior) constraints on the generative model, allowing e.g., for a trade-off between robust control and empirical chance constraint violation. Secondly, we interpret the proposed solution within a message passing framework. Interestingly, the message passing interpretation is not only relevant to the context of ActInf, but also provides a general purpose approach that can account for chance constraints on graphical models. The chance constraint message updates can then be readily combined with other pre-derived message update rules, without the need for custom derivations. The proposed chance-constrained message passing framework thus accelerates the search for workable models in general, and can be used to complement message-passing formulations on generative neural models.
\end{abstract}


\vfill
This is the author's final version of the manuscript, as accepted for publication in MIT Neural Computation.

\clearpage

\section{Introduction}
\label{sec:introduction}
Similar to biological agents, learning to make decisions based on observations and feedback from the environment is also an essential task for autonomous artificial agents. Traditionally, adaptive linear control and model predictive control have been successfully applied in this area \citep{borrelli2017predictive}. Over the past few years, reinforcement learning has become the predominant approach \citep{recht2019tour}. An emerging alternative perspective to decision making under uncertainty is \emph{active inference} (ActInf) \citep{friston_free-energy_2010}. ActInf is a neuroscience-based theory that has been used extensively to explain behavior of biological agents in dynamic environments \citep{friston_free-energy_2010}. 

ActInf is based in the \emph{free energy principle} (FEP), and postulates that perception and action in biological agents minimize a free energy bound on Bayesian surprise. The free energy is an information-theoretic measure that bounds the current and the future expected statistical surprise, i.e., how unpredictable are the observations under a given generative model (GM). The free energy is associated with the Kullback-Leibler (KL) divergence (i.e., the distance) between the approximate and the true posterior. In particular, according to the free energy principle, the agent acts in such a way as to minimize a free-energy bound on the surprise, i.e., Bayesian surprise which, informally speaking, provides a quantification of the difference between the agent's predictions about the system behavior and the observed system behavior. Minimization of free energy is closely related to variational Bayesian methods, reinforcement learning \citep{sallans_using_2001,tschantz2020reinforcement,sajid2021active}, and deep generative models \citep{ueltzhoffer_deep_2018,fountas2020deep}, another set of popular machine learning approaches \citep{goodfellow2014generative}. ActInf is closely related to message passing on graphical models \citep{de2017factor,friston2017graphical}, and several widely used message passing algorithms, including (loopy) belief propagation, variational message passing and expectation propagation can be derived as fixed-point equations of the (Bethe) free energy \citep{heskes_stable_2003,yedidia_constructing_2005,dauwels_variational_2007,zhang2017unifying}. This relation has been harnessed to develop elegant automated methods for ActInf \citep{schwobel_active_2018,van_de_laar_simulating_2019}.

In addition to investigation of motivating connections with the behavior of the biological systems \citep{friston_free_2006,ramstead_answering_2018}, ActInf has been successfully utilized in applications in the traditional stochastic control scenarios, such as linear quadratic Gaussian (LQG) control and similar standard problems such as maze problems \citep{hoffmann_linear_2017,ueltzhoffer_deep_2018,schwobel_active_2018,baltieri_active_2019,millidge2020relationship,imohiosen2020active}, and exploration-exploitation balancing in multi-armed bandit problems \citep{markovic2021empirical}. 

Despite these promising developments, the ActInf framework lacks certain desirable features present in model predictive control. In particular, there is no off-the-shelf standard ActInf formulation that allows inclusion of chance constraints in the problem setting. Chance constraints provide an attractive approach for on-line decision making for uncertain systems \citep{Mesbah_2016}, i.e., systems where the dynamics are not fully known or the system contains certain components that are best modeled in a stochastic manner. In such settings, constraints on the system behavior, such as the agent remaining in a given region of the environment with a given probability, cannot directly be encoded in terms of prior beliefs. In contrast to approaches that constrain all realizations of the random variables, chance constraints allow for a (typically small) probability of constraint violation, which can significantly improve performance since chance constraints enable the decision maker trade performance with probability of constraint violation \citep{BlackmoreOnoWilliams_2011}. 

This paper proposes a computationally tractable approach to chance-constrained decision making, and applies it to an ActInf context. We include chance constraints in the ActInf objective (i.e., the free energy) by using the Lagrangian formalism. We then solve the Lagrangian optimization problem by variational calculus. Finally, we show that the proposed solution not only leads to a modular and scalable message passing framework for ActInf problems, but also provides a general purpose message passing framework that can account for chance constraints on graphical models in general. 
We claim the following main contributions:
\begin{enumerate}
\item We show that the analytic solution to the chance-constrained free energy problem yields posterior beliefs in the form of truncated mixtures. (Theorem \ref{thm:corrected_belief})
\item We show how this solution can be interpreted in terms of message passing on a factor-graph representation of the generative model. (Theorem \ref{thm:message_passing})
\item Consequently, our results provide a message passing framework that is specifically designed to account for chance constraints.
\end{enumerate}

Message passing is inherently modular, and (variational) message update rules can be pre-derived and stored in a lookup table for later use \citep{korl2005factor,van2019automated}. The chance-constrained message updates can then be readily combined with these pre-derived rules, without the need for laborious derivations. Our results illustrate that the proposed framework can successfully find solutions so that the rate of constraint violation specified in the original problem and the one that is actually observed during the closed-loop operation are close. The results also illustrate how to balance the constraints on the actions and the states through the usage of a tuning parameter, which enables exploration of different trade-offs between immediate and delayed intervention.

\section{Problem Statement}
We start by defining a general factorized generative model $f$ with respect to an (arbitrary) collection of variables $x$. As a notational convention, individual variables are indexed by $i, j \in \mathcal{V}$, and factors by $a, b, c \in \mathcal{F}$, unless stated otherwise. The model then factorizes as
\begin{align}
    f(x) = \prod_{a\in\mathcal{F}} f_a(x_a)\,, \label{eq:model}
\end{align}
with non-negative real functions $f_a$, and where $x_a\subset x$ collects the arguments of $f_a$. In a probabilistic generative model, the individual factors usually represent conditional probability distributions. Probabilistic inference is then concerned with obtaining an (approximate) posterior belief $q_j(x_j) \propto \int f(x) \mathrm{d}x_{\setminus j}$ over a variable of interest $x_j$, where $x_{\setminus j}$ indicates the integration over all model variables except $x_j$.

We now briefly recap how the computation of these beliefs can be performed efficiently and automated over a factor graph \citep{loeliger2004signal}, and how this process can be interpreted as a Bethe free energy minimization problem \citep{yedidia_constructing_2005}. With these concepts firmly in place, we move to chance constraints and the formal problem statement in Sec.~\ref{sec:problem}.

\subsection{Factor Graphs for Marginal Belief Computation}
A factor graph can be used to visually represent a factorized function. In this paper we use the bi-partite factor graph representation. A bi-partite factor graph
\begin{align*}
    \mathcal{G} = (\mathcal{F}, \mathcal{V}, \mathcal{E})\,,
\end{align*}
consists of variable-nodes $\mathcal{V}$, factor-nodes $\mathcal{F}$, and edges $\mathcal{E}$ that connect variable-nodes with factor-nodes. A variable-node $i \in \mathcal{V}$ is connected to a factor-node $a\in \mathcal{F}$ by an edge $(i, a) \in \mathcal{E}$ if (and only if) the variable $x_i$ is an argument of the factor-function $f_a$. An example section of a graph is drawn in Fig.\ref{fig:graph_uncorrected}, where the circle and square represent a variable- and factor-node respectively.
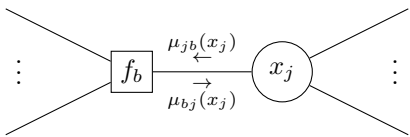
\begin{figure}[h]
    \hfill
    \begin{center}
    \begin{tikzpicture}
        [node distance=20mm,auto,>=stealth']


        \node[smallbox] (f_b) {$f_b$};
        \node[roundbox, right of=f_b] (x_j) {$x_j$};

        \node (x_i_1) at ($(f_b)+(-1.8,0.9)$) {};
        \node (x_i_2) at ($(f_b)+(-1.8,-0.9)$) {};
        \node[left of=f_b, node distance=1.5cm, yshift=1mm] () {$\vdots$};

        \node (x_k_1) at ($(x_j)+(1.8,0.9)$) {};
        \node (x_k_2) at ($(x_j)+(1.8,-0.9)$) {};
        \node[right of=x_j, node distance=1.5cm, yshift=1mm] () {$\vdots$};

        \path[line] (x_i_1) edge[-] (f_b);
        \path[line] (x_i_2) edge[-] (f_b);
        \path[line] (f_b) edge[-] node[anchor=north]{$\substack{\rightarrow\\\mu_{bj}(x_j)}$} node[anchor=south]{$\substack{\mu_{jb}(x_j)\\ \leftarrow}$} (x_j);

        \path[line] (x_j) edge[-] (x_k_1);
        \path[line] (x_j) edge[-] (x_k_2);
    \end{tikzpicture}
    \end{center}
    \caption{Bi-partite subgraph of a model around a variable-node $j$ (circle) and factor-node $b$ (square), with indicated messages. Ellipses represent a continuation of the model.}
    \label{fig:graph_uncorrected}
\end{figure}
We write the neighborhood of a variable-node $i$ as $\mathcal{F}(i)$, which collects all factor-nodes in $\mathcal{F}$ that are direct neighbors of $i$. Similarly, $\mathcal{V}(a)$ collects all variable-nodes in $\mathcal{V}$ that are direct neighbors of $a$.

Suppose we are interested in obtaining a posterior belief $q_j(x_j)$. The belief propagation algorithm \citep{pearl1982reverend} then prescribes we send messages from the branches of the graph towards the variable-node of interest, following the recursive application of the belief propagation update rules:
\begin{subequations}
\label{eq:bp_messages}
\begin{align}
    \mu_{jb}(x_j) &= \prod_{\substack{a\in \mathcal{F}(j)\\a\neq b}} \mu_{aj}(x_j)\\
    \mu_{bj}(x_j) &= \int f_b(x_b) \prod_{\substack{i\in \mathcal{V}(b)\\i\neq j}} \mu_{ib}(x_i) \d{x_{b\setminus j}}\,,
\end{align}
\end{subequations}
where $x_{b\setminus j}$ collects all $x_b$ with the exception of $x_j$. Here, $\mu_{jb}(x_j)$ represents the message from a variable-node $j\in \mathcal{V}$ to a neighboring factor-node $b\in \mathcal{F}(j)$; and reversely for $\mu_{bj}(x_j)$. These messages are illustrated in Fig.~\ref{fig:graph_uncorrected}.
The posterior belief can then be expressed as
\begin{align}
    q_j(x_j) &= \frac{1}{Z_j} \mu_{jb}(x_j) \mu_{bj}(x_j)\,, \label{eq:bp_q_j}    
\end{align}
with $Z_j = \int \mu_{jb}(x_j) \mu_{bj}(x_j) \d{x_j}$ a normalizing constant.

In practice, for numerical stability, messages are often re-normalized after computation. Furthermore, messages are usually scheduled for computation, and are often referred to by their position in the schedule instead of their location in the graph. We will use a similar notation in Sec.~\ref{sec:results}. See \citep{bishop2006pattern} for a more detailed introduction to (approximate) inference on bi-partite graphs.

\subsection{Bethe Free Energy Interpretation}
The Bethe free energy for a factorized model of the form of \eqref{eq:model} is defined as
\begin{align}
    F[q] &= \sum_{a\in \mathcal{F}} U_a[q_a] - \sum_{a\in \mathcal{F}} H[q_a] + (d_i - 1) \sum_{i\in \mathcal{V}} H[q_i]\,, \label{eq:bfe}
\end{align}
where $d_i$ represents the degree of variable $x_i$. Here $U_a[q_a] = -\int q_a(x_a) \log f_a(x_a) \d{x_a}$ denotes the average energy for factor $f_a$, and $H[q_a]=-\int q_a(x_a) \log q_a(x_a) \d{x_a}$ denotes the entropy. The Bethe free energy is optimized with imposed normalization and marginalization constraints:
\begin{subequations}
\label{eq:norm_marg}
\begin{align}
    \int q_a(x_a) \d{x_{a\setminus j}} &= q_j(x_j), \forall a \in \mathcal{F}, \forall j\in \mathcal{V}(a) \label{eq:marg}\\
    \int q_a(x_a) \d{x_a} &= 1, \forall a \in \mathcal{F}\\
    \int q_i(x_i) \d{x_i} &= 1, \forall i \in \mathcal{V}\,,
\end{align}
\end{subequations}
such that the $q_a$ and $q_i$ represent (approximate) posterior probability distributions (beliefs).

\subsection{Free Energy Minimization for Active Inference}
Active Inference usually defines dynamic models that specialize variables into parameters, states, observation and control sequences for past and future times. Free energy minimization for ActInf is then presented as a dual objective, where minimization of free energy for a model of past variables accounts for state and parameter estimation (perception), and free energy minimization of free energy for a model of future variables accounts for policy planning \citep{baltieri2018modularity,van2019application}.

In the present paper we assume that the current state is observed and that model parameters are given. Therefore, this paper only concerns inference for policy planning. Extensions for perception are however straightforward. Chance constraints only affect inference for planning, and therefore standard techniques for state estimation and parameter learning can be employed \citep{van_de_laar_simulating_2019}.

Furthermore, the current paper employs the Bethe Free Energy (BFE) formulation \eqref{eq:bfe} for policy planning \citep{schwobel_active_2018,van_de_laar_simulating_2019} instead of the more traditional Expected Free Energy (EFE) \citep{friston_active_2015}. The BFE is known to lack the epistemic qualities of the EFE \citep{schwobel_active_2018}, which can be compensated for by introducing an additional mutual information term between the states and the observations to the BFE objective \citep{parr2019generalised}. The benefit of the uncompensated BFE however, is that traditional message passing algorithms, including (loopy) belief propagation, variational message passing, expectation propagation and generalized belief propagation algorithms can all be derived as fixed-point equations of the variational free energy by the use of variational calculus, see \citep{yedidia2000generalized,heskes_stable_2003,yedidia_constructing_2005,dauwels_variational_2007,zhang2017unifying}.

\subsection{Chance Constraints}
\label{sec:problem}
A chance constraint imposes that the probability mass of a belief $q_j(x_j), j\in \mathcal{V}$ outside of a `safe' region $\mathcal{S}_j \subset \mathcal{X}_j$ cannot exceed a pre-set threshold $\epsilon \in [0, 1]$. Formally, a chance constraint imposes the inequality
\begin{align}
    1 - \epsilon &\leq \int_{\mathcal{S}_j} q_j(x_j) \d{x_j}\notag\\
    &= \int_{\mathcal{X}_j} q_j(x_j)\, g_j(x_j) \d{x_j}\,, \label{eq:chance}
\end{align}
with
\begin{align*}
    g_j(x_j) =
    \begin{cases}
        1 \text{ if } x_j \in \mathcal{S}_j\\
        0 \text{ otherwise}\,.
    \end{cases}
\end{align*}

Our problem statement then becomes two-fold, namely:
\begin{enumerate}
    \item Find the stationary points of the Bethe free energy \eqref{eq:bfe} under the normalization and marginalization constraints of \eqref{eq:norm_marg} and chance constraints of the form \eqref{eq:chance} (Theorem \ref{thm:corrected_belief});
    \item Interpret the retrieval of stationary points of the chance-constrained Bethe free energy as message passing on a factor graph (Theorem \ref{thm:message_passing}).
\end{enumerate}

The simulations of Sec.~\ref{sec:results} further specialize the model variables into state, observation and control sequences and demonstrate the added value of chance constraints in an ActInf setting. Crucially, with an interpretation of chance constraints in terms of message passing on a factor graph, chance constraints can be readily applied to any factorized model. Formulating chance constraints as a click-on module for approximate inference then greatly improves the application range of chance constraints.

\section{Chance-Constrained Message Passing}
\label{sec:methods}
In this section we formulate the method of chance-constrained message passing. We identify the stationary points of the chance-constrained Bethe free energy and interpret the result in terms of message passing on a factor graph. We work towards a practical message-passing update rule for chance-constrained variables, as summarized in Algorithm~\ref{alg:chance}. A brief introduction to variational calculus is available in Appendix~\ref{app:calculus_of_variations}. Proofs can be found in Appendix~\ref{app:proofs}.

\subsection{Stationary Points}
From the Bethe free energy \eqref{eq:bfe} and the constraints of \eqref{eq:norm_marg}, \eqref{eq:chance}, we can construct the Lagrangian
\begin{align}
    L[q] &= F[q] + \sum_{i\in \mathcal{V}} \gamma_i \left[ \int q_i(x_i) \d{x_i} - 1 \right] + \sum_{a\in \mathcal{F}} \gamma_a \left[ \int q_a(x_a) \d{x_a} - 1 \right] \notag\\
    & + \sum_{a\in \mathcal{F}} \sum_{i\in \mathcal{V}(a)} \int \zeta_{ia}(x_i)\left[q_i(x_i) - \int q_a(x_a) \d{x_{a\setminus i}}\right] \d{x_i} \notag\\
    & + \sum_{i\in \mathcal{V}} \eta_i\left[\int q_i(x_i) g_i(x_i) \d{x_i} - (1 - \epsilon)\right] \label{eq:L_q}\,,
\end{align}
where the Lagrange multipliers $\gamma, \zeta, \eta$ enforce the constraints of \eqref{eq:norm_marg}, \eqref{eq:chance}.

Under strong duality, for the inequality constraint in \eqref{eq:chance}  we have the complementary slackness condition \cite[Ch.~5]{b_boyd}. This condition states that for optimality we have $ \eta_i\left[\int q_i(x_i) g_i(x_i) \d{x_i} - (1 - \epsilon)\right] = 0$. Therefore, either  $\eta_i > 0$, which implies that the chance constraint of \eqref{eq:chance} holds with equality (active) or $\eta_i=0$, which implies that the chance constraint may hold without equality (inactive).
In other words, the complementary slackness condition requires us to  consider two scenarios:  i) \eqref{eq:chance} holds with equality for  $\eta_i >0$  and ii) \eqref{eq:chance} is satisfied under $\eta_i = 0$. 
Hence, if $\eta_i >0$, the chance constraint is activated and enforced with equality.

In Lemmas~\ref{lem:q_b_star}, \ref{lem:q_j_star} we express the stationary points of $L[q]$ in terms of the beliefs. The proofs are presented in Appendix~\ref{pf:lem:q_b_star} and ~Appendix~\ref{pf:lem:q_j_star}. 

\begin{lemma}
    \label{lem:q_b_star}
    Stationary points of \eqref{eq:L_q} as a functional of $q_b, b\in \mathcal{F}$, are of the form
    \begin{align}
        q_b^*(x_b) &= \frac{1}{Z_b} f_b(x_b) \prod_{i\in \mathcal{V}(b)} \mu_{ib}(x_i)\,, \label{eq:q_b_star_result}
    \end{align}
    with
    \begin{align*}
        Z_b &= \int f_b(x_b) \prod_{i\in \mathcal{V}(b)} \mu_{ib}(x_i) \d{x_b}
    \end{align*}
    a normalizing constant.
\end{lemma}
\begin{proof}
See Appendix~\ref{pf:lem:q_b_star}.
\end{proof}
Note that the $\mu_{ib}$ have not yet been identified or interpreted as messages. We will explicitly make this connection in Sec.~\ref{sec:cc_mp}.

\begin{lemma}
    \label{lem:q_j_star}
    Stationary points of \eqref{eq:L_q} as a functional of $q_j, j\in \mathcal{V}$, are of the form
    \begin{align}
        q_j^*(x_j; \eta_j) &= \frac{1}{Z_j(\eta_j)} \exp\!\left(-\eta_j g_j(x_j)\right) \prod_{a\in \mathcal{F}(j)} \mu_{aj}(x_j)\,, \label{eq:q_j_star_result}
    \end{align}
    with
    \begin{align*}
        Z_j(\eta_j) &= \int \exp\!\left(-\eta_j g_j(x_j)\right) \prod_{a\in \mathcal{F}(j)} \mu_{aj}(x_j) \d{x_j}
    \end{align*}
    a normalizer that still depends on $\eta_j$.
\end{lemma}
\begin{proof}
See Appendix~\ref{pf:lem:q_j_star}.
\end{proof}
Note that, in contrast to \eqref{eq:bp_q_j}, this result incorporates an additional exponential term for $\eta_j$. We will identify this multiplier in Sec.~\ref{sec:active_cc}. However, we  already know that when the chance constraint for $j$ is inactive, hence $\eta_j=0$ as a consequence of the complementary slackness condition. In this case, \eqref{eq:q_j_star_result} reduces to \eqref{eq:bp_q_j}.

\subsection{Active Chance Constraint}
\label{sec:active_cc}
In this section, we identify the stationary points under active chance constraint. The result is stated in Theorem~\ref{thm:corrected_belief}.  

\begin{theorem}
    \label{thm:corrected_belief}
    Under active chance constraint, stationary points of \eqref{eq:L_q} as a functional of $q_j, j\in \mathcal{V}$ are of the form
    \begin{align}
        q_j^*(x_j; \eta_j=\eta_j^*) &=
        \begin{cases}
            \frac{1 - \epsilon}{\Phi^{(0)}_j} q^{(0)}_j(x_j) &\text{ if } x_j \in \mathcal{S}_j\\
            \frac{\epsilon}{1 - \Phi^{(0)}_j} q^{(0)}_j(x_j) &\text{ otherwise,}
        \end{cases} \label{eq:optimal_correction}
    \end{align}
    with
    \begin{subequations}
    \label{eq:q_phi_0}
    \begin{align}
        q^{(0)}_j(x_j) &= q_j^*(x_j; \eta_j=0)\,, \label{eq:q_0}\\
        \Phi^{(0)}_j &= \int_{\mathcal{S}_j} q^{(0)}_j(x_j) \d{x_j}\,, \label{eq:phi_0}\\
        \eta_j^* &= \log (\epsilon \Phi^{(0)}_j) - \log (1 - \epsilon) - \log (1 - \Phi^{(0)}_j)\,.
    \end{align}
    \end{subequations}
\end{theorem}
\begin{proof}
See Appendix~\ref{pf:thm:corrected_belief}. 
\end{proof}

This remarkable result tells us that the corrected belief $q_j^*(x_j; \eta_j=\eta_j^*)$ is obtained by scaling the probability mass of the uncorrected belief $q_j^{(0)}(x_j)$ over the respective safe and unsafe regions. This defines the corrected belief as a mixture of truncated beliefs. The optimal scaling of \eqref{eq:optimal_correction} ensures that the overflow is equal to $\epsilon$.

The complementary slackness condition ensures that the chance constraint is only enforced if the probability mass of the \emph{unconstrained} belief overflows the `safe' region $\mathcal{S}_j$ by more than $\epsilon$; i.e., the uncorrected belief is `unsafe' when
\begin{align}
    \epsilon &< 1 - \Phi^{(0)}_j\,, \label{eq:is_unsafe}
\end{align}
where we refer to $\Phi^{(0)}_j$ as the `safe mass'.

If \eqref{eq:is_unsafe} is satisfied, then the posterior density $q^{(0)}_j(x_j)$ is corrected according to \eqref{eq:optimal_correction}, which `pushes' the probability mass (just) back inside the safe region.

\subsection{Chance-Constrained Message Passing}
\label{sec:cc_mp}
In this section, we show that chance constraints \eqref{eq:optimal_correction} can be interpreted as auxiliary factor-nodes (with a specific node-function), and can be enforced by belief propagation in an augmented graph.

\begin{theorem}
    \label{thm:message_passing}
    Given a bipartite graph $\mathcal{G} = (\mathcal{F}, \mathcal{V}, \mathcal{E})$ with a variable node $j \in \mathcal{V}$, and an associated Bethe free energy \eqref{eq:bfe} with a chance constraint \eqref{eq:chance} on the belief $q_j(x_j)$. Then, stationary points of \eqref{eq:L_q} can be obtained by belief propagation on an augmented graph $\mathcal{G}' = (\mathcal{F}', \mathcal{V}, \mathcal{E}')$, where
    \begin{subequations}
    \label{eq:augmented_graph}
    \begin{align}
        \mathcal{F}' &= \mathcal{F} \cup g\\
        \mathcal{E}' &= \mathcal{E} \cup (j, g)\,,
    \end{align}
    \end{subequations}
    and auxiliary node function
    \begin{align}
        \label{eq:auxiliary_node_function}
        f_g(x_j) &=
        \begin{cases}
            \frac{1 - \epsilon}{\Phi^{(0)}_j} &\text{ if } x_j \in \mathcal{S}_j\\
            \frac{\epsilon}{1 - \Phi^{(0)}_j} &\text{ otherwise.}
        \end{cases}
    \end{align}
\end{theorem}
\begin{proof}
See Appendix~\ref{pf:thm:message_passing}. 
\end{proof}
Theorem \ref{thm:message_passing} shows that chance-constrained message passing can be seamlessly incorporated within the belief propagation framework. Chance constraints simply enter the model definition as auxiliary factors, whose factor function depends upon the incoming message, see Fig.~\ref{fig:graph}. Because uncorrected belief \eqref{eq:q_0} is being represented by the (re-normalized) incoming message $\mu_{jg}(x_j)$, this allows for a modular application of chance constraints by augmenting the original graphical model with auxiliary nodes.
 
\begin{figure}[h]
    \hfill
    \begin{center}
    \begin{tikzpicture}
        [node distance=20mm,auto,>=stealth']


        \node[smallbox] (f_b) {$f_b$};
        \node[roundbox, right of=f_b] (x_j) {$x_j$};
        \node[smallbox, dashed, above of=x_j, node distance=17mm] (f_g) {$f_g$};

        \node (x_i_1) at ($(f_b)+(-1.8,0.9)$) {};
        \node (x_i_2) at ($(f_b)+(-1.8,-0.9)$) {};
        \node[left of=f_b, node distance=1.5cm, yshift=1mm] () {$\vdots$};

        \node (x_k_1) at ($(x_j)+(1.8,0.9)$) {};
        \node (x_k_2) at ($(x_j)+(1.8,-0.9)$) {};
        \node[right of=x_j, node distance=1.5cm, yshift=1mm] () {$\vdots$};

        \path[line] (x_i_1) edge[-] (f_b);
        \path[line] (x_i_2) edge[-] (f_b);
        \path[line] (f_b) edge[-] node[anchor=north]{$\substack{\rightarrow\\\mu_{bj}(x_j)}$} node[anchor=south]{$\substack{\mu_{jb}(x_j)\\ \leftarrow}$} (x_j);

        \path[line] (x_j) edge[-] (x_k_1);
        \path[line] (x_j) edge[-] (x_k_2);
        \path[dashed] (f_g) edge[-] node[anchor=east, pos=0.35]{$_{\mu_{gj}(x_j) \downarrow}$} node[anchor=west, pos=0.35]{$_{\uparrow \mu_{jg}(x_j)}$} (x_j);
    \end{tikzpicture}
    \end{center}
    \caption{Bi-partite graph around a chance-constrained variable $x_j$, with indicated auxiliary factor $f_g$ (dashed square) and messages. Ellipses represent the continued model by an arbitrary (possibly zero) number of connected edges.}
    \label{fig:graph}
\end{figure}
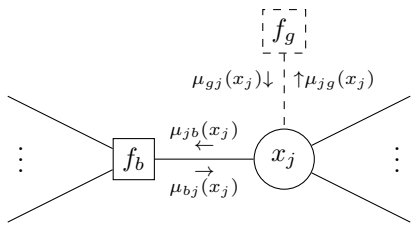

\subsection{Gaussian Approximation}
Since the message $\mu_{gj}(x_j)$ introduces discontinuities, the computations for dependent messages may grow prohibitively complex. For efficient computations, it can be helpful to make a Gaussian approximation $\tilde{q}_j(x_j)$ to the corrected belief $q_j^*(x_j; \eta_j=\eta_j^*)$, e.g., by moment matching. The resulting (approximate) message then follows from
\begin{align*}
    \mu_{gj}(x_j) = \tilde{q}^{(n)}_j(x_j)/\mu_{jg}(x_j)\,.
\end{align*}
If the message $\mu_{jg}(x_j)$ is also Gaussian, this computation is easily performed by subtracting the canonical statistics. This procedure then resembles the expectation propagation algorithm \citep{minka2001expectation,cox2018robust}. Interestingly, the expectation propagation algorithm can also be derived in terms of Bethe free energy optimization, where the marginalization constraints \eqref{eq:marg} are replaced by moment-matching constraints \citep{zhang2017unifying}. This makes the Gaussian approximation consistent with the Lagrangian approach as presented in this paper.

The approximated belief $\tilde{q}_j(x_j)$ however renders the chance constraint \eqref{eq:chance} inexact. As a result, the approximated belief needs to be iteratively re-corrected:
\begin{align}
    q^{(n)}_j(x_j) &=
    \begin{cases}
        \frac{1 - \epsilon}{\Phi^{(n-1)}_j} \tilde{q}^{(n-1)}_j(x_j) &\text{ if } x_j \in \mathcal{S}_j\\
        \frac{\epsilon}{1 - \Phi^{(n-1)}_j} \tilde{q}^{(n-1)}_j(x_j) &\text{ otherwise,}
    \end{cases} \label{eq:iterative_correction}
\end{align}
where $n$ denotes an iteration counter. This leads to the procedure summarized in Alg.~\ref{alg:chance}, and depicted in Fig.~\ref{fig:curve}.

\begin{algorithm}
\caption{Chance-constrained message passing with Gaussian approximation}
\label{alg:chance}
\begin{algorithmic}
\STATE {Given a Gaussian inbound message $\mu_{jg}(x_j)$}
\STATE {Compute the uncorrected belief $q^{(0)}_j(x_j)$ through \eqref{eq:q_0}}
\STATE {Compute the safe mass $\Phi^{(0)}_j$ through \eqref{eq:phi_0}}
\STATE {Initialize the approximated belief $\tilde{q}^{(0)}_j(x_j) = q^{(0)}_j(x_j)$}
\STATE {Initialize the iteration counter $n = 0$}
\WHILE {$\epsilon + \delta < 1 - \Phi^{(n)}_j$}
    \STATE {\% \emph{Chance constraint is violated with some tolerance $\delta$}}
    \STATE {Increase the counter $n \leftarrow n+1$}
    \STATE {Compute the corrected belief $q^{(n)}_j(x_j)$ through \eqref{eq:iterative_correction}}
    \STATE {Approximate $\tilde{q}^{(n)}_j(x_j) \approx q^{(n)}_j(x_j)$ by Gaussian moment matching}
    \STATE {Compute $\Phi^{(n)}_j = \int_{\mathcal{S}_j} \tilde{q}^{(n)}_j(x_j) \d{x_j}$, the safe mass of the approximated belief}
\ENDWHILE
\RETURN {The message $\mu_{gj}(x_j) = \tilde{q}^{(n)}_j(x_j)/\mu_{jg}(x_j)$}
\end{algorithmic}
\end{algorithm}

\begin{figure}[h]
    \hfill
    \begin{center}
    \pgfmathdeclarefunction{gauss}{2}{%
      \pgfmathparse{1/(#2*sqrt(2*pi))*exp(-((x-#1)^2)/(2*#2^2))}%
    }

    \pgfmathdeclarefunction{phi}{3}{%
      \pgfmathparse{1-1/(1 + exp(-0.07056*((#1-#2)/#3)^3 - 1.5976*(#1-#2)/#3))}%
    }

    \begin{tikzpicture}
    \begin{axis}[
      no markers, domain=0:5, samples=100,
      axis lines*=left, xlabel=$x_j$,
      every axis x label/.style={at=(current axis.right of origin),anchor=west},
      height=4.5cm, width=10cm,
      xtick={2}, xticklabels=\empty, ytick=\empty,
      enlargelimits=false, clip=false, axis on top,
      grid = major
      ]
      \addplot [fill=black!20, draw=none, domain=2:5] {gauss(3,1)} \closedcycle;
      \addplot [very thick,black] {gauss(3,1)};

      \draw [yshift=-0.4cm, latex-latex](axis cs:2,0) -- node [fill=white] {$\mathcal{S}_j$} (axis cs:5,0);
      \node [anchor=south] (q_tilde_n_min) at (axis cs:1.0,0.11) {$\tilde{q}^{(n-1)}_j(x_j)$};
      \node [anchor=south] (phi_n_min) at (axis cs:3.15,0.11) {$\Phi^{(n-1)}_j$};
    \end{axis}
    \end{tikzpicture}
    \begin{tikzpicture}
    \begin{axis}[
      no markers, domain=0:5, samples=100,
      axis lines*=left, xlabel=$x_j$,
      every axis x label/.style={at=(current axis.right of origin),anchor=west},
      height=4.5cm, width=10cm,
      xtick={2}, xticklabels=\empty, ytick=\empty,
      enlargelimits=false, clip=false, axis on top,
      grid = major
      ]
      \addplot [fill=black!20, draw=none, domain=0:2] {0.07/(1-phi(2,3,1))*gauss(3,1)} \closedcycle;
      \addplot [very thick, black, domain=0:2] {(0.07/(1-phi(2,3,1)))*gauss(3,1)};
      \addplot [very thick, black, domain=2:5] {((1-0.07)/phi(2,3,1))*gauss(3,1)};
      \addplot [very thick, black, dashed] {gauss(3.22,0.85)};

      \draw [yshift=-0.4cm, latex-latex](axis cs:2,0) -- node [fill=white] {$\mathcal{S}_j$} (axis cs:5,0);
      \node [anchor=south] (eps) at (axis cs:1.75,0.01) {$\epsilon$};
      \node [anchor=south] (q_n) at (axis cs:3.9,0.11) {$q^{(n)}_j(x_j)$};
      \node [anchor=south] (q_tilde_n) at (axis cs:4.3,0.35) {$\tilde{q}^{(n)}_j(x_j)$};
    \end{axis}
    \path (current bounding box.north) ++ (0,0.5cm); 
    \end{tikzpicture}
    \end{center}
    \caption{Example of beliefs as computed by Algorithm~\ref{alg:chance}. The top figure evaluates the probability mass within the ``safe'' zone. The bottom figure applies the correction (solid curve) and approximates the corrected belief by Gaussian moment matching (dashed curve).}
    \label{fig:curve}
\end{figure}
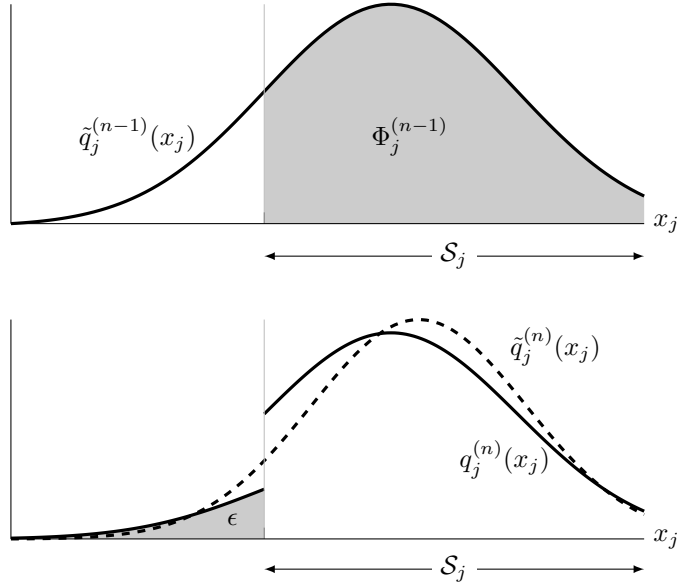

With this algorithm, we have derived a practical chance-constrained message update from the first principles. The message update can be readily applied to any continuous variable that requires a chance constraint. Note however, that when multiple chance constraints are imposed on the model, the message passing algorithm itself becomes an iterative procedure because of circular message dependencies. For example, a message incoming to an auxiliary node $g$ might (indirectly) depend on a message that exits another auxiliary node $h$. In turn, this exiting message depends on the incoming message to $h$ \eqref{alg:chance}, which depends on the message exiting $g$, etcetera. In order to break this circular message dependency, uninformative messages can be used to initialize the algorithm.

\section{Simulations}
\label{sec:results}
In this section we simulate a drone that aims to elevate itself above a given height threshold with a preset probability, under the influence of a stochastic vertical wind. We define the drone elevation level over time by $x = \{x_0, \dots, x_t, \dots, x_L\}, x_t\in \mathbb{R}$, and actions (ascension velocity) $a = \{a_0, \dots, a_t, \dots, a_L\}, a_t\in \mathbb{R}$. A time-dependent $m_{w,t}$ defines the expected wind velocity that acts upon the agent. The discrete-time stochastic system is defined as:
\begin{align*}
    w_t &\sim \N{m_{w,t}, v_w}\\
    x_{t+1} &= x_t + a_t + w_t\,,
\end{align*}
where $v_w$ defines the wind velocity variance.

We define an agent that directly observes its elevation level and has knowledge of the statistical system properties $m_{w,t}$ and $v_w$. The agent models future states of the system with a fixed time horizon $T$. As a shorthand notation, we write the future (including current) states $\overline{x}_t = \{x_t, \dots, x_{t+T}\}$ and control variables $\overline{u}_t = \{u_t, \dots, u_{t+T-1}\}$. For notational convenience, we drop the $t$ subscript from these collections. The agent model at time $t$ is defined as:
\begin{align}
    f_t(\overline{x}, \overline{u}) = \prod_{k=t}^{t+T-1} p_{x,k}(x_{k+1}|u_k, x_k) p_u(u_k)\,, \label{eq:agent_model}
\end{align}
with a respective state transition model and control prior
\begin{subequations}
\label{eq:agent_state_control}
\begin{align}
    p_{x,k}(x_{k+1}|u_k, x_k) &= \N{x_{k+1} | x_k + u_k + m_{w,k}, v_w}\\
    p_u(u_k) &= \N{u_k | 0, \lambda^{-1}}\,.
\end{align}
\end{subequations}

We factorize and constrain the variational posterior distribution such that \citep{van_de_laar_simulating_2019}
\begin{align}
    q_t(\overline{x}_{\setminus t}, \overline{u}) = q_t(\overline{x}_{\setminus t}) \prod_{k=t}^{t+T-1} \delta(u_k - a_k)\,, \label{eq:variational_posterior}
\end{align}
where $\overline{x}_{\setminus t}$ indicates the collection of latent states (the state sequence $\overline{x}$ without the observed current state $x_t$). The goal of the agent controller then becomes to find the policy $\pi_t = \{a_t, \dots, a_{t+T-1}\}$ that minimizes the Bethe free energy
\begin{align}
    F[q_t; x_t, \pi_t] &= \idotsint q_t(\overline{x}_{\setminus t}, \overline{u}) \log \frac{q_t(\overline{x}_{\setminus t}, \overline{u})}{f_t(\overline{x}, \overline{u})} \d{\overline{x}_{\setminus t}} \d{\overline{u}}\,, \label{eq:system_bfe}
\end{align}
under the normalization and marginalization constraints of \eqref{eq:norm_marg} and chance constraints
\begin{align*}
    1 - \epsilon \leq \int_{\mathcal{S}} q_{x,k}(x_k) \d{x_k},\, \forall k \in \{t+1, \dots, t+T\}\,,
\end{align*}
where the safe region $\mathcal{S}=(1, \infty)$ and violation probability $\epsilon$ are identical for all future state variables.

\subsection{Graphical Model and Schedule}
As detailed in Sec.~\ref{sec:methods}, Bethe free energy minimization under chance constraints can be performed by message passing on an augmented model. The graphical representation of the augmented model is depicted in Fig.~\ref{fig:augmented_model}.

\begin{figure}[h]
    \hfill
    \begin{center}
    \begin{tikzpicture}
        [node distance=15mm,auto,>=stealth']

        \node[roundbox, fill=darkgrey] (x_t) {};
            \node[yshift=-6mm] at (x_t) {$x_t$};

        \node[smallbox, right of=x_t, node distance=3cm] (p_x_k) {};
            \node[yshift=-5mm] at (p_x_k) {$p_{x,k}$};
        \node[roundbox, above of=p_x_k, node distance=2cm] (u_k) {};
            \node[xshift=6mm, yshift=3mm] at (u_k) {$u_k$};
        \node[roundbox, fill=darkgrey, above left of=p_x_k, node distance=2cm] (m_w_k) {};
            \node[yshift=7mm] at (m_w_k) {$m_{w,k}$};
        \node[roundbox, fill=darkgrey, above right of=p_x_k, node distance=2cm] (v_w_k) {};
            \node[xshift=5mm, yshift=5mm] at (v_w_k) {$v_w$};
        \node[smallbox, above of=u_k, node distance=12mm] (p_u_k) {};
            \node[xshift=5mm] at (p_u_k) {$p_u$};
        \node[roundbox, fill=darkgrey, above of=p_u_k, node distance=12mm] (lambda_k) {};
            \node[xshift=6mm] at (lambda_k) {$\lambda$};
        \node[roundbox, right of=p_x_k] (x_k_plus_1) {};
            \node[xshift=5mm, yshift=6mm] at (x_k_plus_1) {$x_{k+1}$};
        \node[smallbox, dashed, below of=x_k_plus_1] (f_x_k_plus_1) {};
            \node[xshift=-8mm, yshift=1mm] at (f_x_k_plus_1) {$f_{x,k+1}$};

        \path[line] (x_t) edge[-] (p_x_k); 
        \path[line] (u_k) edge[-] (p_x_k);
        \path[line] (m_w_k) edge[-] (p_x_k); 
        \path[line] (v_w_k) edge[-] (p_x_k); 
        \path[line] (p_u_k) edge[-] (u_k); 
        \path[line] (lambda_k) edge[-] (p_u_k); 
        \path[line] (p_x_k) edge[-] (x_k_plus_1); 
        \path[dashed] (x_k_plus_1) edge[-] (f_x_k_plus_1);

        \node[right of=x_k_plus_1, node distance=2cm] (ellipsis) {$\dots$};

        \path[line] (x_k_plus_1) edge[-] (ellipsis); 

        \draw[rounded corners] (0.75,-2.0) rectangle (5.75,5.0);
        \node at (1.7,-1.7) {$_{k=t:t+T-1}$};
    \end{tikzpicture}
    \end{center}
    \caption{Augmented graphical representation of the agent model \eqref{eq:agent_model}. Circles and squares indicate variable- and factor-nodes respectively. Auxiliary factor-nodes \eqref{eq:auxiliary_node_function} are dashed, and dark circles indicate observed variables or fixed parameters. Ellipses indicate a continuation of the framed section until the lookahead time horizon.}
    \label{fig:augmented_model}
\end{figure}
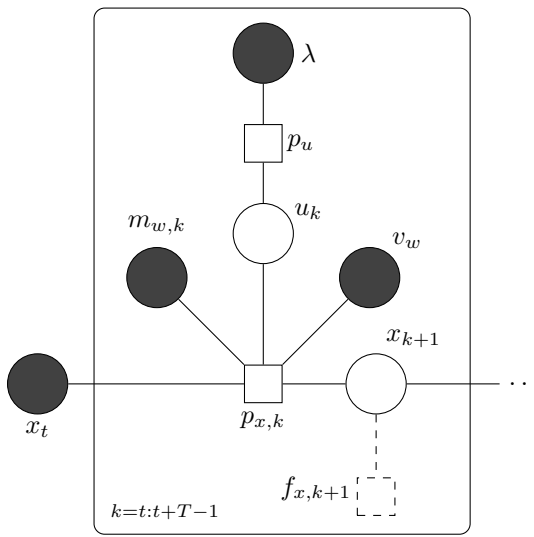

\begin{figure}[h]
    \hfill
    \begin{center}    
    \begin{tikzpicture}
        [node distance=15mm,auto,>=stealth']


        \node (ellipses_k) {\dots};
        \node[mediumbox, right of=ellipses_k, node distance=18mm] (u_add_k) {$+$};
        \node[roundbox, right of=u_add_k] (aux_k) {};
        \node[mediumbox, right of=aux_k] (m_add_k) {$+$};
        \node[roundbox, right of=m_add_k, node distance=18mm] (x_k_plus_1) {};
            \node[yshift=7mm] at (x_k_plus_1) {$x_{k+1}$};
        \node[right of=x_k_plus_1] (ellipses_k_plus_1) {\dots};

        \node[mediumbox, below of=x_k_plus_1, dashed] (f_x_k_plus_1) {};
            \node[xshift=9mm] at (f_x_k_plus_1) {$f_{x,k+1}$};
        \node[roundbox, fill=darkgrey, below of=m_add_k] (m_w_k) {};
            \node[xshift=-10mm] at (m_w_k) {$m_{w,k}$};
        \node[roundbox, above of=u_add_k] (aux_u) {};
        \node[mediumbox, above of=aux_u] (aux_n) {$\mathcal{N}$};
        \node[roundbox, fill=darkgrey, left of=aux_n] (v_w_k) {};
            \node[yshift=6mm] at (v_w_k) {$v_w$};
        \node[roundbox, above of=aux_n] (u_k) {};
            \node[xshift=-7mm] at (u_k) {$u_k$};
        \node[mediumbox, right of=u_k] (p_u_k) {$\mathcal{N}$};
        \node[roundbox, fill=darkgrey, right of=p_u_k] (lambda_k) {};
            \node[xshift=6mm] at (lambda_k) {$\lambda$};

        \path[line] (ellipses_k) edge[-] (u_add_k);
        \path[line] (u_add_k) edge[-] (aux_k);
        \path[line] (aux_k) edge[-] (m_add_k);
        \path[line] (m_add_k) edge[-] (x_k_plus_1);
        \path[line] (x_k_plus_1) edge[-] (ellipses_k_plus_1);

        \path[dashed] (x_k_plus_1) edge[-] (f_x_k_plus_1);
        \path[line] (m_add_k) edge[-] (m_w_k);
        \path[line] (aux_u) edge[-] (u_add_k);
        \path[line] (aux_n) edge[-] (aux_u);
        \path[line] (u_k) edge[-] (aux_n);
        \path[line] (p_u_k) edge[-] (u_k);
        \path[line] (lambda_k) edge[-] (p_u_k);
        \path[line] (v_w_k) edge[-] (aux_n);

        \msg{up}{right}{ellipses_k}{u_add_k}{0.3}{1}
        \msg{left}{down}{u_k}{aux_n}{0.4}{2}
        \darkmsg{left}{down}{aux_n}{aux_u}{0.5}{3}
        \msg{up}{right}{u_add_k}{aux_k}{0.5}{4}
        \msg{up}{right}{m_add_k}{x_k_plus_1}{0.45}{5}
        \msg{right}{up}{x_k_plus_1}{f_x_k_plus_1}{0.65}{6}
        \msg{up}{right}{x_k_plus_1}{ellipses_k_plus_1}{0.65}{7}

        \msg{down}{left}{x_k_plus_1}{ellipses_k_plus_1}{0.65}{A}
        \msg{left}{down}{x_k_plus_1}{f_x_k_plus_1}{0.65}{B}
        \msg{down}{left}{m_add_k}{x_k_plus_1}{0.45}{C}
        \msg{up}{left}{aux_k}{m_add_k}{0.5}{D}
        \msg{left}{up}{aux_u}{u_add_k}{0.5}{E}
        \darkmsg{right}{up}{u_k}{aux_n}{0.4}{F}
        \msg{up}{left}{u_k}{p_u_k}{0.5}{G}
        \msg{down}{left}{ellipses_k}{u_add_k}{0.3}{H}

        \draw[dashed] (0.9,-0.5) rectangle (5.3,3.5);
        \node at (4.8,3.1) {$p_{x,k}$};
    \end{tikzpicture}
    \end{center}
    \caption{Augmented agent model \eqref{eq:agent_model}, with $p_{x,k}$ expanded according to \eqref{eq:agent_state_control} (dashed rectangle), and indicated forward (numbers) and backward (letters) message passing schedules for optimization of \eqref{eq:system_bfe}. Circle and square nodes indicate variable- and factor-nodes respectively. Dark nodes indicate observed variables or fixed parameters, and auxiliary factor-nodes \eqref{eq:auxiliary_node_function} are dashed. Ellipses indicate a continuation of the model. Dark messages are computed by the variational update rule, see \citep{winn2005variational,dauwels_variational_2007}.}
    \label{fig:schedule}
\end{figure}

The schedule comprises a forward-backward scheme, as illustrated in Fig.~\ref{fig:schedule}.
Four message updates in Fig.~\ref{fig:schedule} are of particular interest. Firstly, since \eqref{eq:variational_posterior} constrains the belief over controls to a point-mass, it follows that
\begin{align*}
    \mu_{\smallcircled{2}}^{(i)}(u_k) &= \delta(u_k - a_k^{(i-1)})\,,
\end{align*}
where $i$ counts the number of schedule (forward-backward) iterations. The schedule is initialized with $a_k^{(0)} = 0$ for all $k \geq t$. Secondly, $\mu_{\smallcircled{B}}^{(i)}(x_{k+1})$ takes on the role of $\mu_{jg}(x_j)$ in Alg.~\ref{alg:chance}. Because the noise in the model is Gaussian, this message will be an (unnormalized) Gaussian as well. Therefore, by application of Alg.~\ref{alg:chance}, the third message of interest, $\mu_{\smallcircled{6}}^{(i)}(x_{k+1})$ is computed. For the initial forward pass, $\mu_{\smallcircled{B}}^{(0)}(x_{k+1}) = 1$ is considered uninformative. Fourthly, $\mu_{\smalldarkcircled{F}}^{(i)}(u_k)$ carries information upward to the control variables. Because the variational posterior is chosen to factorize between the state and control sequence \eqref{eq:variational_posterior}, the $\mu_{\smalldarkcircled{F}}^{(i)}(u_k)$ message is computed by a variational update rule as detailed in \citep{winn2005variational} and \citep{dauwels_variational_2007}.

The action for the next iteration then follows from
\begin{align*}
    q_{u,k}^{(i)}(u_k) &\propto \mu_{\smalldarkcircled{F}}^{(i)}(u_k)\, \mu_{\smallcircled{G}}^{(i)}(u_k)\\
    a_k^{(i)} &= \operatorname{mode} q_{u,k}^{(i)}(u_k)\,.
\end{align*}
Iterating the schedule then corresponds with an expectation maximization scheme. The expectation step of this scheme computes the $\mu_{\smalldarkcircled{F}}^{(i)}(u_k)$ message from the actions $a_k^{(i-1)}$. The maximization step then chooses the updated actions $a_k^{(i)}$ as the current MAP-estimate of $u_k$. The schedule is iterated until the policy converges.

Message passing simulations\footnote{Source code for the simulations is available for download at \url{http://biaslab.github.io/materials/cc_simulations.zip}} are performed with the ForneyLab probabilistic programming toolbox \citep{cox2019factor}, version 0.11.3.

\subsection{Control Law}
Note that the Bethe free energy of \eqref{eq:system_bfe} is still a function of the observed current elevation $x_t$. We can then evaluate the optimal action $a_t$ as a function of the current elevation $x_t$ (the control law), for a given wind profile, chance constraint and model parameters. In order to gain an intuition for controller behavior, we fix $m_{w,t}=0$ for all $t$. We plot the control law in Fig.~\ref{fig:control_law}, for varying values of the lookahead horizon $T$, chance constraint threshold $\epsilon$, wind variance $v_w$ and control prior precision $\lambda$.

\begin{figure}[h]
    \hfill
    \begin{center}
        \resizebox{0.8\columnwidth}{!}{\includegraphics{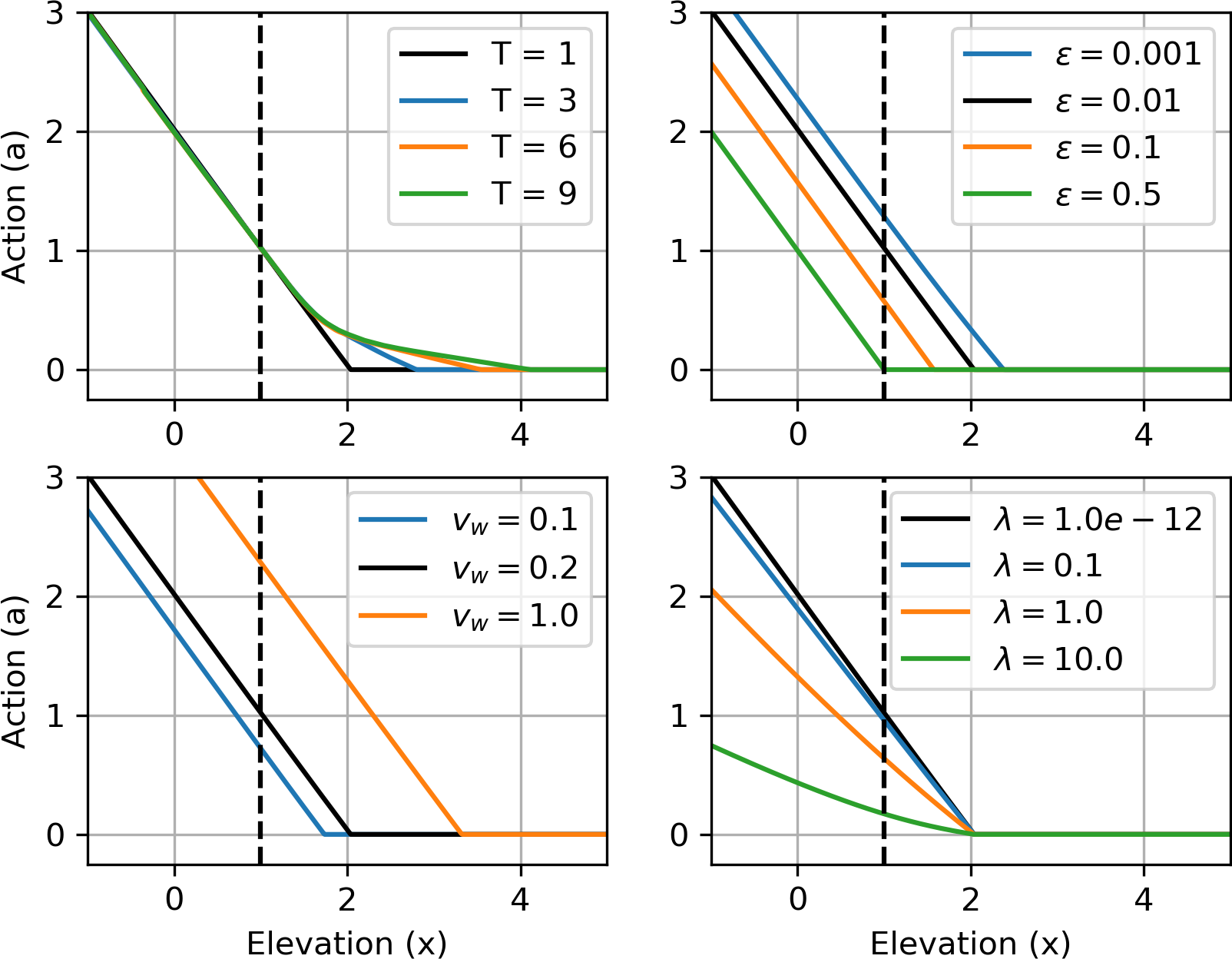}}
    \end{center}
    \caption{Slices of the control law for $m_{w,t}=0, \mathcal{S}=(1, \infty)$, varied around reference setting $T=1, \epsilon=0.01, v_w=0.2, \lambda=10^{-12}$ (black curves). Dashed vertical lines indicate the minimal safe elevation.}
    \label{fig:control_law}
\end{figure}

The top-left diagram shows that with growing lookahead horizon $T$, the agent starts intervening at higher elevation. With this anticipatory effect the agent prepares for events in the more distant future. The top-right diagram also shows that the agent intervenes at higher elevation with decreasing $\epsilon$. When violation of the constraint grows less desirable, the agent must intervene earlier in order to assure that sufficient probability mass is present in the safe region. Also note that no further action is proposed beyond an intervention threshold. Once the agent is sufficiently elevated, no corrections are proposed until the agent wanders (or is forced) below the intervention threshold. The bottom-left figure shows a similar effect for growing wind velocity variance $v_w$. When the system grows more stochastic, chance constraint abidance is ensured by intervening at higher elevations. Finally, the bottom-right figure illustrates what happens when the chance constraint is combined with a Gaussian prior constraint on control. Increasing the control prior precision $\lambda$ penalizes immediate correction. For low precisions (low penalty on control magnitude), the slope of the control law below the intervention threshold is equal to $1$, and compensation is immediate. Control grows more robust with growing precision, at the cost of prolonged chance constraint violation.

\subsection{Comparison Against a Goal-Driven Agent}
In order to illustrate the difference in behavior between a chance- and a goal-driven ActInf agent, we compare the results of Fig.~\ref{fig:control_law} with an ActInf agent where the chance constraint is replaced by a goal prior. We use the graphical model definition of Fig.~\ref{fig:augmented_model} and define the auxiliary node function as a fixed prior $f_{x, k+1}(x_{k+1}) = \N{x_{k+1} | m_x, \vartheta_x}$ for all $t \leq k \leq t+T-1$. We choose $m_x=2$, and the variance $\vartheta_x=0.18478$ such that the overflow of the safe region $1 - \int_{\mathcal{S}} f_{x, k+1}(x_{k+1}) \d{x_{k+1}} \approx 0.01$ resembles the situation for $\epsilon=0.01$. The message passing schedule then follows the definition of Fig.~\ref{fig:schedule}, where $\mu_{\smallcircled{6}}^{(i)}(x_{k+1})$ is no longer computed by Alg.~\ref{alg:chance} and propagates the fixed goal prior instead. Fig.~\ref{fig:control_law_reference} shows the resulting control law for $m_{w,t}=0, T=1, v_w=0.2$ and varying $\lambda$.

The results of Fig.~\ref{fig:control_law_reference} show that the control for the goal-driven agent grows more robust with increasing $\lambda$ -- similar to the control law for the chance-driven agent (Fig.~\ref{fig:control_law}, bottom right). For the smallest $\lambda$, the control law for the prior-driven agent resembles the corresponding control law for the chance-driven agent (dotted curve) only for elevations $x < 2$. For elevations $x > 2$, the goal-driven agent proposes downward corrections, while the chance-driven agent proposes no corrections. This comparison illustrates how a chance-driven agent avoids unnecessary interventions.

\begin{figure}[h]
    \hfill
    \begin{center}
        \resizebox{0.5\columnwidth}{!}{\includegraphics{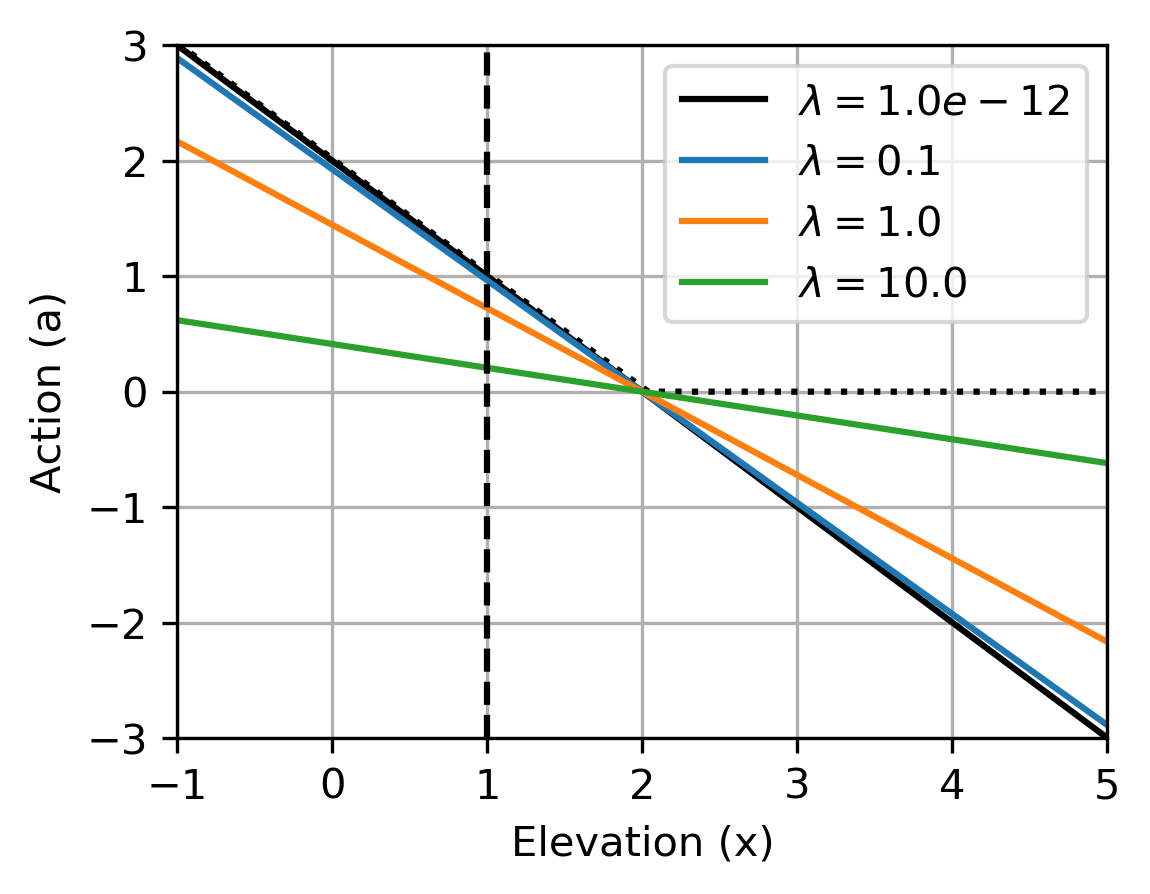}}
    \end{center}
    \caption{Slices of the control law for a goal-driven agent with $m_{w,t}=0, T=1, m_x=2, \vartheta_x=0.18478, v_w=0.2$ with varying $\lambda$. The dashed vertical line indicates the minimal safe elevation. The black dotted curve represents the reference result ($\lambda=10^{-12}$) for the chance-driven agent (Fig.~\ref{fig:control_law}, black curves).}
    \label{fig:control_law_reference}
\end{figure}

\subsection{Simulation Results}
In this section we study an active inference agent in interaction with a simulated environment. The action-perception loop is based on \citep{van_de_laar_simulating_2019} and consists of four steps at every time $t$:
\begin{enumerate}
\item \emph{Observe} the current agent elevation;
\item \emph{Infer} a policy from the current elevation and the future expected wind velocities by chance-constrained message passing;
\item \emph{Act} by selecting the first (current) action from the inferred policy;
\item \emph{Execute} the selected action in the system and advance the time index by one.
\end{enumerate}

\begin{figure}[h]
    \hfill
    \makebox[\textwidth][c]{
    \begin{subfigure}[t]{0.65\textwidth}
        \centering
        \includegraphics[width=\textwidth]{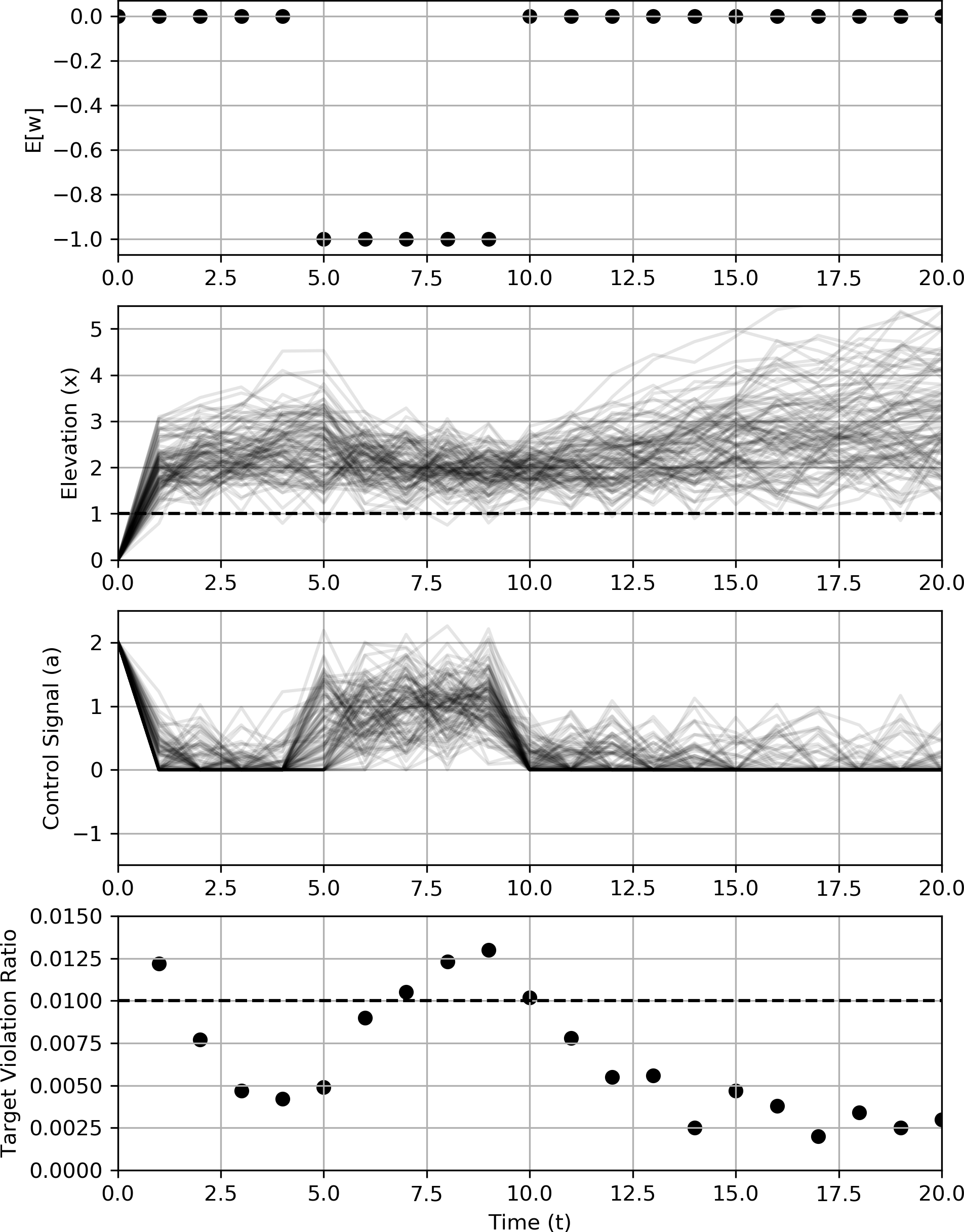}
    \end{subfigure}
    \begin{subfigure}[t]{0.65\textwidth}
        \centering
        \includegraphics[width=\textwidth]{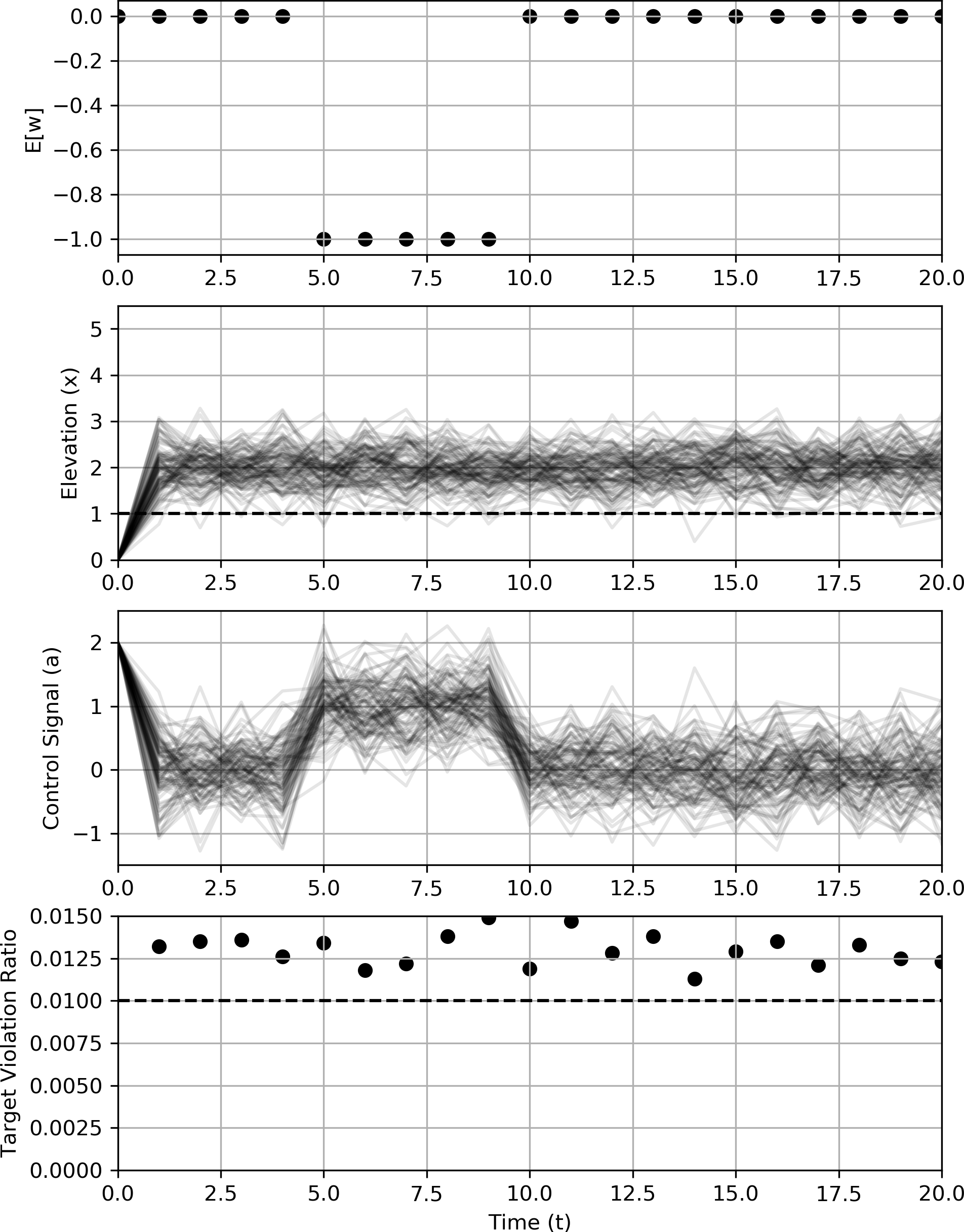}
    \end{subfigure}}
    \caption{Results for ten thousand simulations with varying wind strength over time, and $T=1, v_w=0.2, \lambda=10^{-12}$, for a chance-driven agent ($\epsilon=0.01$, left), and a goal-driven agent ($m_x=2, \vartheta_x=0.18478$, right).}
    \label{fig:simulation}
\end{figure}

The results for ten thousand independent runs are plotted in Fig.~\ref{fig:simulation} for a chance-driven agent (left) and a goal-driven agent (right). The first row of diagrams plots the expected wind velocity over time, which is identical for each run. The sampled wind velocity trajectories $w_t$ do vary per run, under influence of the wind velocity variance $v_w$. For $5 \leq t < 10$ a downward draft attempts to push the drone below the minimal safe elevation (dashed). The second row plots the drone elevation trajectory for a randomly selected subset of runs. Corresponding actions are plotted in the third row. The fourth row evaluates the relative number of runs that violate the safe-zone over time.

It can be seen that both agents undertake corrective actions in order to compensate for the downward wind. However, while the chance-driven agent (left) only proposes upward corrections below the intervention threshold, the goal-driven agent (right) proposes additional downward corrections above the threshold. Furthermore, it can be seen that the maximal empirical violation for the chance-constrained agent mostly remains below the chance constraint target violation probability of $\epsilon=0.01$ (dashed), while the goal-driven agent systematically overshoots the target violation probability, i.e. violates the chance constraint. Compared to the chance-driven agent, the maximal empirical violations for the goal-driven agent are also larger. This effect can be explained in terms of the constrained beliefs. Namely, the chance-driven agent constrains the posterior beliefs, while the goal-driven agent imposes prior constraints on the model. Prior constraints may still be violated by the corresponding posterior beliefs, leading to more pronounced empirical violations.

\section{Conclusions}
\label{sec:conclusion}
In this paper, we formulated chance-constrained optimization of the Bethe free energy in terms of message passing on a factor graph. We showed that, in the factor graph representation of the generative model, chance constraints can be imposed by auxiliary factors that force (a specified portion of) the probability mass of the chance-constrained beliefs inside a designated safe-zone. Message passing on the augmented graph, with the auxiliary factor-nodes included in the graph, then automatically balances the imposed chance constraints with additional (prior) constraints on the generative model. Chance constraints can thus be interpreted as modular click-on extensions to the generative model, similar to conventional factor-nodes \citep{loeliger2004signal}, and can thus be used to complement message-passing formulations on generative neural models \citep{friston2017graphical,van2018forneylab}.

However, because the analytical result for the chance-constrained update includes an inherent discontinuity, direct application of this rule may still lead to message updates that grow prohibitively complex. To remedy this, we proposed an algorithm that approximates the resulting message with a Gaussian form. This algorithm offers a tractable formulation of chance-constrained message passing. The proposed message passing interpretation of chance constraints then vastly enhances the modularity and flexibility of chance-constrained inference, and can accelerate the search for workable models \citep{blei2014build}.

We demonstrated chance-constrained message passing in the context of active inference. We compared the simulated behavior of a chance-driven agent with a goal-driven agent, where the chance constraints are replaced by traditional prior beliefs on future outcomes. The results illustrate how the goal-driven agent continually proposes corrections, whereas the chance-driven agent seizes interventions above a threshold. Chance-constrained ActInf may thus avoid unnecessary interventions and reduce the cost of control.

The results for the chance-driven agent showed that, in the absence of additional prior constraints, the empirical chance constraint violation ratio mostly remains below the pre-set target violation probability. An added prior constraint on controls robustifies control at the cost of prolonged chance constraint violation. Chance-constrained active inference thus weights all imposed constraints on the generative model, allowing e.g., for a trade-off between robust control and empirical chance constraint violation.

\subsection*{Acknowledgments}
This work was supported, in part, by GN Hearing A/S and the Swedish Research Council (under Grants 2015-04011 and 2018-03701).

\appendix
\section*{Appendix}
\section{Calculus of Variations}
\label{app:calculus_of_variations}
The calculus of variations offers a principled method for optimizing functionals (a function of a function that returns a scalar). We follow \citep{engel2013density} and consider the impact of a variation in a function $q(x), x \in \mathcal{X}$, on a functional $L[q]$. We define an infinitesimal variation of $q$ by
\begin{align*}
    \delta q \overset{\Delta}{=} \beta \phi\,, 
\end{align*}
where $\beta \rightarrow 0$, and $\phi(x)$ is a continuous and differentiable ``test'' function.

The functional derivative $\delta L/\delta q$ relates a variation in $q$ to a change in $L$, by \citep{parr1980density}:
\begin{align}
    \frac{\d{L[q + \beta \phi]}}{\d{\beta}}\bigg|_{\beta=0} &= \int \frac{\delta L}{\delta q}(x)\, \phi(x) \d{x}\,. \label{eq:first_variation}
\end{align}
The procedure then becomes to apply the operations on the l.h.s. to $L$, and bring it into the form of the r.h.s., which allows us to identify the functional derivative $\delta L/\delta q$. The stationary points $q^*$ are then obtained by setting $\delta L/\delta q \overset{!}{=} 0$ and solving for $q$.

\section{Proofs}
\label{app:proofs}

\subsection{Proof of Lemma \ref{lem:q_b_star}}\label{pf:lem:q_b_star}
Application of \eqref{eq:first_variation} to \eqref{eq:L_q} as a functional of $q_b$, yields
\begin{align*}
    &\frac{\d{L[q_b + \beta \phi_b]}}{\d{\beta}}\bigg|_{\beta=0} = \int \phi_b(x_b)\bigg[\log \frac{q_b(x_b)}{f_b(x_b)} + 1 + \gamma_b - \sum_{i\in \mathcal{V}(b)} \zeta_{ib}(x_i)\bigg] \d{x_b}\,.
\end{align*}
Identifying the functional derivative $\delta L[q_b]/\delta q_b$ and setting it to zero, we obtain
\begin{align}
    q_b^*(x_b) &= f_b(x_b)\exp\!\bigg[\sum_{i\in \mathcal{V}(b)} \zeta_{ib}(x_i) - \gamma_b - 1\bigg]\,.
\end{align}
We now define $\mu(x_j) = \exp \zeta(x_j)$ and apply the normalization constraint, which recovers \eqref{eq:q_b_star_result}. \hfill \qedsymbol

\subsection{Proof of Lemma \ref{lem:q_j_star}}\label{pf:lem:q_j_star}
Application of \eqref{eq:first_variation} to \eqref{eq:L_q} as a functional of $q_j$, yields
\begin{align*}
    &\frac{\d{L[q_j + \beta \phi_j]}}{\d{\beta}}\bigg|_{\beta=0} = \int \phi_j(x_j)\bigg[ -(d_j - 1) + \gamma_j\\
    & - (d_j - 1)\log q_j(x_j) + \sum_{a\in \mathcal{F}(j)} \zeta_{ja}(x_j) + \eta_j g_j(x_j)\bigg] \d{x_j}\,.
\end{align*}
Identifying the functional derivative $\delta L[q_j]/\delta q_j$ and setting it to zero, yields
\begin{align}
    q_j^*(x_j) &= \exp\!\bigg[\frac{1}{d_j-1}\bigg(1 - d_j + \gamma_j + \sum_{a\in \mathcal{F}(j)} \zeta_{ja}(x_j) + \eta_j g_j(x_j)\bigg)\bigg]\,, \label{eq:q_j_star_var}
\end{align}
which is the first expression for $q_j^*$.

We can obtain a second expression for $q_j^*$ by applying the marginalization constraint to the result of Lemma \ref{lem:q_b_star}. Substituting \eqref{eq:q_b_star_result} in \eqref{eq:marg},
\begin{align}
    q_j^*(x_j) &= \int q_b^*(x_b) \d{x_{b\setminus j}}\notag \\
    &= \frac{1}{Z_b}\mu_{jb}(x_j) \overbrace{\int f_b(x_b) \prod_{\substack{i\in \mathcal{V}(b)\\i\neq j}}\mu_{ib}(x_i) \d{x_{b\setminus j}}}^{\mu_{bj}(x_j)}\,, \label{eq:q_j_star_marg}
\end{align}
where we identified a new quantity $\mu_{bj}(x_j)$ (note the reverse indexing).

Interestingly, the marginalization result of \eqref{eq:q_j_star_marg} not only holds for the specific factor $b$, but for all factors that neighbor $j$. Therefore, by symmetry, we can iterate the relation of \eqref{eq:q_j_star_marg} for all $c \in \mathcal{F}(j)$:
\begin{align*}
    \prod_{c\in \mathcal{F}(j)} q_j^*(x_j) &= \prod_{c\in \mathcal{F}(j)} \frac{1}{Z_c}\mu_{jc}(x_j) \mu_{cj}(x_j)\,.        
\end{align*}
We choose to exclude $b$ itself from the iteration on both sides, and obtain
\begin{align}
    \prod_{\substack{c\in \mathcal{F}(j)\\c\neq b}} q_j^*(x_j) &= \prod_{\substack{c\in \mathcal{F}(j)\\c\neq b}} \frac{1}{Z_c} \mu_{jc}(x_j) \mu_{cj}(x_j)\,. \label{eq:prod_d_j_min_one}
\end{align}

We substitute \eqref{eq:q_j_star_var} in the l.h.s. of \eqref{eq:prod_d_j_min_one}, and note that the product on the l.h.s. now has $d_j-1$ terms, and that neither of these terms depend on $c$. This allows us to remove the $d_j-1$ terms from the exponent of \eqref{eq:q_j_star_var}, which yields
\begin{align*}
    &\exp\!\left(1 - d_j + \gamma_j + \eta_j g_j(x_j)\right)\!\! \prod_{a\in \mathcal{F}(j)} \mu_{ja}(x_j) = \prod_{\substack{c\in \mathcal{F}(j)\\c\neq b}} \frac{1}{Z_c} \mu_{jc}(x_j) \mu_{cj}(x_j)\,.
\end{align*}

Canceling duplicate terms and simplifying, we obtain an expression for $\mu_{jb}$ as identified in \eqref{eq:q_j_star_marg}:
\begin{align}
    \mu_{jb}(x_j) &\propto \exp\!\left(-\eta_j g_j(x_j)\right) \prod_{\substack{a\in \mathcal{F}(j)\\a\neq b}} \mu_{aj}(x_j)\,. \label{eq:backward_message}
\end{align}

Finally, substituting \eqref{eq:backward_message} back in \eqref{eq:q_j_star_marg} and re-normalizing, we recover \eqref{eq:q_j_star_result}.
\hfill \qedsymbol

\subsection{Proof of Theorem \ref{thm:corrected_belief}}\label{pf:thm:corrected_belief}
We start from \eqref{eq:q_j_star_result}, and use the definitions of \eqref{eq:q_phi_0} to obtain
\begin{align*}
    \int_{\mathcal{S}_j} q_j^*(x_j; \eta_j) \d{x_j} = \frac{\Phi^{(0)}_j\exp(-\eta_j)}{\Phi^{(0)}_j\exp(-\eta_j) - \Phi^{(0)}_j + 1}\,,
\end{align*}
which leads to
\begin{align*}
    \exp(-\eta_j^*) = \frac{(1 - \epsilon)(1 - \Phi^{(0)}_j)}{\epsilon \Phi^{(0)}_j}\,.
\end{align*}
We have now identified the $\eta_j$ multiplier. Substituting this result back in \eqref{eq:q_j_star_result} recovers \eqref{eq:optimal_correction}, which expresses the \emph{corrected} belief in terms of the \emph{uncorrected} belief.
\hfill \qedsymbol

\subsection{Proof of Theorem \ref{thm:message_passing}}\label{pf:thm:message_passing}
From \eqref{eq:q_j_star_result}, we express the uncorrected belief in terms of the messages
\begin{align}
    q_j^{(0)}(x_j) &= \frac{1}{Z_j^{(0)}} \prod_{a\in \mathcal{F}(j)} \mu_{aj}(x_j)\,, \label{eq:q_j_star_0}
\end{align}
with $Z_j^{(0)} = Z_j(\eta_j=0)$.

We now construct the augmented graph $\mathcal{G}' = (\mathcal{F}', \mathcal{V}, \mathcal{E}')$ according to \eqref{eq:augmented_graph}, and define a message
\begin{align}
    \mu_{gj}(x_j) = f_g(x_j)\,, \label{eq:factor_update_corrected}
\end{align}
with $f_g(x_j)$ as defined by \eqref{eq:auxiliary_node_function}.

Substituting \eqref{eq:q_j_star_0} and \eqref{eq:factor_update_corrected} in \eqref{eq:optimal_correction} then yields the corrected belief in terms of the messages
\begin{align}
    q_j^*(x_j; \eta_j=\eta_j^*) = \frac{1}{Z_j^{(0)}} \mu_{gj}(x_j) \mu_{jg}(x_j)\,, \label{eq:q_j_corrected}
\end{align}
with
\begin{align}
    \mu_{jg}(x_j) = \prod_{\substack{a\in \mathcal{F}'(j)\\a\neq g}} \mu_{aj}(x_j)\,. \label{eq:variable_update_corrected}
\end{align}

The results of \eqref{eq:factor_update_corrected}, \eqref{eq:q_j_corrected} and \eqref{eq:variable_update_corrected} can be interpreted as belief propagation \eqref{eq:bp_messages}, \eqref{eq:bp_q_j} on the augmented graph $\mathcal{G}'$.
\hfill \qedsymbol

\bibliographystyle{apalike}
\bibliography{bibliography}

\begin{thebibliography}{}

\bibitem[Baltieri and Buckley, 2018]{baltieri2018modularity}
Baltieri, M. and Buckley, C.~L. (2018).
\newblock The modularity of action and perception revisited using control
  theory and active inference.
\newblock In {\em Artificial life conference proceedings}, pages 121--128. MIT
  Press.

\bibitem[Baltieri and Buckley, 2019]{baltieri_active_2019}
Baltieri, M. and Buckley, C.~L. (2019).
\newblock Active {Inference}: {Computational} {Models} of {Motor} {Control}
  without {Efference} {Copy}.
\newblock In {\em 2019 {Conf.} on {Cognitive} {Computational} {Neuroscience}}.

\bibitem[Bishop, 2006]{bishop2006pattern}
Bishop, C.~M. (2006).
\newblock {\em Pattern recognition and machine learning}.
\newblock Springer.

\bibitem[{Blackmore} et~al., 2011]{BlackmoreOnoWilliams_2011}
{Blackmore}, L., {Ono}, M., and {Williams}, B.~C. (2011).
\newblock Chance-constrained optimal path planning with obstacles.
\newblock {\em IEEE Transactions on Robotics}, 27(6):1080--1094.

\bibitem[Blei, 2014]{blei2014build}
Blei, D.~M. (2014).
\newblock Build, compute, critique, repeat: Data analysis with latent variable
  models.
\newblock {\em Annual Review of Statistics and Its Application}, 1:203--232.

\bibitem[Borrelli et~al., 2017]{borrelli2017predictive}
Borrelli, F., Bemporad, A., and Morari, M. (2017).
\newblock {\em Predictive control for linear and hybrid systems}.
\newblock Cambridge University Press.

\bibitem[Boyd and Vandenberghe, 2004]{b_boyd}
Boyd, S. and Vandenberghe, L. (2004).
\newblock {\em Convex Optimization}.
\newblock Cambridge University Press.

\bibitem[Cox and de~Vries, 2018]{cox2018robust}
Cox, M. and de~Vries, B. (2018).
\newblock Robust expectation propagation in factor graphs involving both
  continuous and binary variables.
\newblock In {\em 2018 26th European Signal Processing Conference (EUSIPCO)},
  pages 2583--2587. IEEE.

\bibitem[Cox et~al., 2019]{cox2019factor}
Cox, M., van~de Laar, T.~W., and de~Vries, B. (2019).
\newblock A factor graph approach to automated design of {Bayesian} signal
  processing algorithms.
\newblock {\em International Journal of Approximate Reasoning}, 104:185--204.

\bibitem[Dauwels, 2007]{dauwels_variational_2007}
Dauwels, J. (2007).
\newblock On {Variational} {Message} {Passing} on {Factor} {Graphs}.
\newblock In {\em {IEEE} {Inter.} {Symp.} on {Information} {Theory}}, pages
  2546--2550.

\bibitem[de~Vries and Friston, 2017]{de2017factor}
de~Vries, B. and Friston, K.~J. (2017).
\newblock A factor graph description of deep temporal active inference.
\newblock {\em Frontiers in computational neuroscience}, 11:95.

\bibitem[Engel and Dreizler, 2013]{engel2013density}
Engel, E. and Dreizler, R.~M. (2013).
\newblock {\em Density functional theory}.
\newblock Springer.

\bibitem[Fountas et~al., 2020]{fountas2020deep}
Fountas, Z., Sajid, N., Mediano, P.~A., and Friston, K. (2020).
\newblock Deep active inference agents using monte-carlo methods.
\newblock {\em arXiv preprint arXiv:2006.04176}.

\bibitem[Friston et~al., 2015]{friston_active_2015}
Friston, K., Rigoli, F., Ognibene, D., Mathys, C., Fitzgerald, T., and Pezzulo,
  G. (2015).
\newblock Active inference and epistemic value.
\newblock {\em Cognitive Neuroscience}, 6(4):187--214.

\bibitem[Friston, 2010]{friston_free-energy_2010}
Friston, K.~J. (2010).
\newblock The free-energy principle: a unified brain theory?
\newblock {\em Nature Reviews Neuroscience}, 11(2):127--138.

\bibitem[Friston et~al., 2006]{friston_free_2006}
Friston, K.~J., Kilner, J., and Harrison, L. (2006).
\newblock A free energy principle for the brain.
\newblock {\em Journal of Physiology, Paris}, 100(1-3):70--87.

\bibitem[Friston et~al., 2017]{friston2017graphical}
Friston, K.~J., Parr, T., and de~Vries, B. (2017).
\newblock The graphical brain: belief propagation and active inference.
\newblock {\em Network Neuroscience}, 1(4):381--414.

\bibitem[Goodfellow et~al., 2014]{goodfellow2014generative}
Goodfellow, I.~J., Pouget-Abadie, J., Mirza, M., Xu, B., Warde-Farley, D.,
  Ozair, S., Courville, A.~C., and Bengio, Y. (2014).
\newblock Generative adversarial nets.
\newblock In {\em NIPS}.

\bibitem[Heskes, 2003]{heskes_stable_2003}
Heskes, T. (2003).
\newblock Stable fixed points of loopy belief propagation are local minima of
  the bethe free energy.
\newblock In {\em Advances in neural information processing systems}, pages
  359--366.

\bibitem[Hoffmann and Rostalski, 2017]{hoffmann_linear_2017}
Hoffmann, C. and Rostalski, P. (2017).
\newblock Linear {Optimal} {Control} on {Factor} {Graphs} - a {Message}
  {Passing} {Perspective}.
\newblock In {\em 20th IFAC World Congress}, Toulouse, France.

\bibitem[Imohiosen et~al., 2020]{imohiosen2020active}
Imohiosen, A., Watson, J., and Peters, J. (2020).
\newblock Active inference or control as inference? a unifying view.
\newblock In {\em 1st International Workshop on Active Inference}.

\bibitem[Korl, 2005]{korl2005factor}
Korl, S. (2005).
\newblock {\em A factor graph approach to signal modelling, system
  identification and filtering}.
\newblock ETH Zurich.

\bibitem[Loeliger et~al., 2004]{loeliger2004signal}
Loeliger, H.-A., Dauwels, J., Koch, V.~M., and Korl, S. (2004).
\newblock Signal processing with factor graphs: examples.
\newblock In {\em First International Symposium on Control, Communications and
  Signal Processing, 2004.}, pages 571--574. IEEE.

\bibitem[Markovic et~al., 2021]{markovic2021empirical}
Markovic, D., Stojic, H., Schwoebel, S., and Kiebel, S.~J. (2021).
\newblock An empirical evaluation of active inference in multi-armed bandits.
\newblock {\em arXiv preprint arXiv:2101.08699}.

\bibitem[{Mesbah}, 2016]{Mesbah_2016}
{Mesbah}, A. (2016).
\newblock Stochastic model predictive control: An overview and perspectives for
  future research.
\newblock {\em IEEE Control Systems Magazine}, 36(6):30--44.

\bibitem[Millidge et~al., 2020]{millidge2020relationship}
Millidge, B., Tschantz, A., Seth, A.~K., and Buckley, C.~L. (2020).
\newblock On the relationship between active inference and control as
  inference.
\newblock In {\em 1st International Workshop on Active Inference}.

\bibitem[Minka, 2001]{minka2001expectation}
Minka, T.~P. (2001).
\newblock Expectation propagation for approximate {Bayesian} inference.
\newblock In {\em Proceedings of the Seventeenth conference on Uncertainty in
  artificial intelligence}, pages 362--369.

\bibitem[Parr, 1980]{parr1980density}
Parr, R.~G. (1980).
\newblock Density functional theory of atoms and molecules.
\newblock In {\em Horizons of Quantum Chemistry}, pages 5--15. Springer.

\bibitem[Parr and Friston, 2019]{parr2019generalised}
Parr, T. and Friston, K.~J. (2019).
\newblock Generalised free energy and active inference.
\newblock {\em Biological cybernetics}, 113(5):495--513.

\bibitem[Pearl, 1982]{pearl1982reverend}
Pearl, J. (1982).
\newblock Reverend bayes on inference engines: A distributed hierarchical
  approach.
\newblock In {\em Proc. of the Second AAAI Conference on Artificial
  Intelligence}, AAAI'82, page 133–136.

\bibitem[Ramstead et~al., 2018]{ramstead_answering_2018}
Ramstead, M. J.~D., Badcock, P.~B., and Friston, K.~J. (2018).
\newblock Answering {Schr{\"o}dinger}'s question: {A} free-energy formulation.
\newblock {\em Physics of Life Reviews}.

\bibitem[Recht, 2019]{recht2019tour}
Recht, B. (2019).
\newblock A tour of reinforcement learning: The view from continuous control.
\newblock {\em Annual Review of Control, Robotics, and Autonomous Systems},
  2:253--279.

\bibitem[Sajid et~al., 2021]{sajid2021active}
Sajid, N., Ball, P.~J., Parr, T., and Friston, K.~J. (2021).
\newblock Active inference: demystified and compared.
\newblock {\em Neural Computation}, 33(3):674--712.

\bibitem[Sallans and Hinton, 2001]{sallans_using_2001}
Sallans, B. and Hinton, G.~E. (2001).
\newblock Using free energies to represent {Q}-values in a multiagent
  reinforcement learning task.
\newblock In {\em Adv. in neural information process. systems}, pages
  1075--1081.

\bibitem[Schw\"obel et~al., 2018]{schwobel_active_2018}
Schw\"obel, S., Kiebel, S., and Markovic, D. (2018).
\newblock Active {Inference}, {Belief} {Propagation}, and the {Bethe}
  {Approximation}.
\newblock {\em Neural Computation}, 30(9):2530--2567.

\bibitem[Tschantz et~al., 2020]{tschantz2020reinforcement}
Tschantz, A., Millidge, B., Seth, A.~K., and Buckley, C.~L. (2020).
\newblock Reinforcement learning through active inference.
\newblock {\em arXiv preprint arXiv:2002.12636}.

\bibitem[Ueltzh\"offer, 2018]{ueltzhoffer_deep_2018}
Ueltzh\"offer, K. (2018).
\newblock Deep {Active} {Inference}.
\newblock {\em Biological Cybernetics}, 112(6):547--573.

\bibitem[van~de Laar, 2019]{van2019automated}
van~de Laar, T.~W. (2019).
\newblock {\em Automated design of Bayesian signal processing algorithms}.
\newblock Eindhoven University of Technology.

\bibitem[van~de Laar et~al., 2018]{van2018forneylab}
van~de Laar, T.~W., Cox, M., Senoz, I., Bocharov, I., and de~Vries, B. (2018).
\newblock Forneylab: a toolbox for biologically plausible free energy
  minimization in dynamic neural models.
\newblock In {\em Conference on Complex Systems}.

\bibitem[van~de Laar and de~Vries, 2019]{van_de_laar_simulating_2019}
van~de Laar, T.~W. and de~Vries, B. (2019).
\newblock Simulating {Active} {Inference} {Processes} by {Message} {Passing}.
\newblock {\em Frontiers in Robotics and AI}, 6:20.

\bibitem[van~de Laar et~al., 2019]{van2019application}
van~de Laar, T.~W., {\"O}z{\c{c}}elikkale, A., and Wymeersch, H. (2019).
\newblock Application of the free energy principle to estimation and control.
\newblock {\em arXiv preprint arXiv:1910.09823}.

\bibitem[Winn and Bishop, 2005]{winn2005variational}
Winn, J. and Bishop, C.~M. (2005).
\newblock Variational message passing.
\newblock {\em Journal of Machine Learning Research}, 6(Apr):661--694.

\bibitem[Yedidia et~al., 2005]{yedidia_constructing_2005}
Yedidia, J.~S., Freeman, W., and Weiss, Y. (2005).
\newblock Constructing free-energy approximations and generalized belief
  propagation algorithms.
\newblock {\em IEEE Transactions on Information Theory}, 51(7):2282--2312.

\bibitem[Yedidia et~al., 2000]{yedidia2000generalized}
Yedidia, J.~S., Freeman, W.~T., Weiss, Y., et~al. (2000).
\newblock Generalized belief propagation.
\newblock In {\em NIPS}, volume~13, pages 689--695.

\bibitem[Zhang et~al., 2017]{zhang2017unifying}
Zhang, D., Wang, W., Fettweis, G., and Gao, X. (2017).
\newblock Unifying message passing algorithms under the framework of
  constrained {Bethe} free energy minimization.
\newblock {\em arXiv preprint arXiv:1703.10932}.

\end{thebibliography}

\end{document}